%% file: 0.main.tex
\documentclass{article}

\usepackage[nonatbib, preprint]{neurips_2021}
\usepackage[utf8]{inputenc} 
\usepackage[T1]{fontenc}    
\usepackage{hyperref}       
\usepackage{url}            
\usepackage{booktabs}       
\usepackage{amsfonts}       
\usepackage{nicefrac}       
\usepackage{microtype}      
\usepackage{comment}
\usepackage{float}
\usepackage{graphicx}
\usepackage{amsmath}
\usepackage{amsthm}
\usepackage{subfigure}
\usepackage{enumitem}
\usepackage{caption}
\usepackage[dvipsnames]{xcolor}
\usepackage[sorting = none, backend = bibtex, style=numeric-comp,maxbibnames=99,firstinits=true]{biblatex}
\hypersetup{colorlinks,linkcolor={blue},citecolor={blue},urlcolor={Maroon}}

\newtheorem{theorem}{Theorem}[section]
\newtheorem*{theorem*}{Theorem}

\newtheorem{lemma}[theorem]{Lemma}

\usepackage[linesnumbered, ruled, vlined]{algorithm2e} 

\usepackage{algorithmicx}
\usepackage{multirow}

\usepackage{tikz}

\title{DeCOM: Decomposed Policy for Constrained Cooperative Multi-Agent Reinforcement Learning}

\author{
	Zhaoxing Yang\thanks{yiannis@sjtu.edu.cn},
	Rong Ding, 
	Haiming Jin, 
	Yifei Wei, 
	Haoyi You, 
	Guiyun Fan, \\
	\textbf{Xiaoying Gan}, 
	\textbf{Xinbing Wang}\\
	{Shanghai Jiao Tong University}\\
}

\begin{document}
	
	\maketitle
	
	\begin{abstract}
		In recent years, multi-agent reinforcement learning (MARL) has presented impressive performance in various applications. However, physical limitations, budget restrictions, and many other factors usually impose \textit{constraints} on a multi-agent system (MAS), which cannot be handled by traditional MARL frameworks. Specifically, this paper focuses on constrained MASes where agents work \textit{cooperatively} to maximize the expected team-average return under various constraints on expected team-average costs, and develops a \textit{constrained cooperative MARL} framework, named DeCOM, for such MASes. In particular, DeCOM decomposes the policy of each agent into two modules, which empowers information sharing among agents to achieve better cooperation. In addition, with such modularization, the training algorithm of DeCOM separates the original constrained optimization into an unconstrained optimization on reward and a constraints satisfaction problem on costs. DeCOM then iteratively solves these problems in a computationally efficient manner, which makes DeCOM highly scalable. We also provide theoretical guarantees on the convergence of DeCOM's policy update algorithm. Finally, we validate the effectiveness of DeCOM with various types of costs in both toy and large-scale (with 500 agents) environments.
	\end{abstract}
	
	\section{Introduction}\label{introduction}
	\input{1.introduction.tex}
	
	\section{Constrained Cooperative Markov Game}
	\input{2.preliminary.tex}

	\section{DeCOM Framework}\label{sec:framework}
	\input{3.framework.tex}
	
	\section{Training Algorithm}\label{sec:training_alg}
	\input{4.training_alg.tex}

	\section{Experiments}
	\input{5.experiments.tex}

	\section{Related Works}
	\input{6.related_works.tex}
	
	\section{Conclusions and Discussions}\label{conclusion}
	In this paper, we propose a novel constrained cooperative MARL framework, named DeCOM, which facilitates agent cooperation by empowering information sharing among agents. By iteratively solving the unconstrained optimization problem on reward and the constrains satisfaction problem on costs, DeCOM learns policies in a scalable, efficient, and easy-to-implement manner. Experiment results in four simulation environments, including CTC-safe, CTC-fair, CDSN and CLFM, validate DeCOM's effectiveness.
	
	
	
	\newpage
	\printbibliography
	
	\newpage
	\section{Proofs}

\input{appendix_proof}

	\section{Experimental Details and Results}
	\input{appendix_exp}
\end{document}

%% file: 1.introduction.tex
Recent years have seen great success of {\textit{multi-agent reinforcement learning (MARL)} in unconstrained {\textit{multi-agent system (MAS)}}, such as video games \cite{NEURIPS2018_7fea637f, pmlr-v97-jaques19a, Vinyals2019GrandmasterLI, Berner2019Dota2W, baker2020emergent}, and many others. However, in practice, an MAS usually works under various constraints introduced by physical limitations, budget restrictions, as well as requirements on certain performance metrics. For example, to avoid collisions, robot swarms have to keep their distances from obstacles and between each other above a threshold  \cite{luo2020multirobot}. As another example, the fairness of power consumption among sensors has to be maintained above a certain level for the sustainability of sensor networks \cite{xu2020learning}.
	
	In practice, many constrained MASes are \textit{cooperative} in nature, where agents cooperatively maximize the team-average return under constraints on certain types of team-average costs. Such cooperation exists in the above robot swarms and sensor networks, as well as other real-world scenarios, such as managing a fleet of ridesharing vehicles \cite{efficient_fm, coride, tmc} where the unfairness among drivers' incomes has to be upper bounded for sufficient driver satisfaction. Inevitably, such joint requirement of cooperation and constraint satisfaction calls for new decision-making mechanisms. Therefore, in this paper, we aim to \textit{develop an MARL framework specifically tailored for constrained cooperative MASes}. 

	One intuitive solution is to directly extend existing single-agent constrained reinforcement learning \cite{achiam2017constrained,tessler2018reward,ly1,pmlr-v97-le19a} to our multi-agent setting, by utilizing a centralized controller to compute the joint actions of all agents. However, it is hard to scale such approach to MASes with a large number of agents. Instead, we achieve scalability by adopting the \textit{centralized training decentralized execution} framework \cite{foerster2016learning}, where each agent is equipped with a local policy that makes decisions without the coordination from any central controller. 
	
	However, it is usually challenging for decentralized decision making to achieve cooperation. To address this challenge, we propose a novel constrained cooperative MARL framework, named DeCOM\footnote{The name DeCOM comes from \underline{De}composed policy for \underline{C}onstrained c\underline{O}operative \underline{M}ARL.}, which facilitates agent cooperation by appropriate information sharing among them. Specifically, DeCOM decomposes an agent's local policy into a \textit{base policy} and a \textit{perturbation policy}, where the former outputs the agent's  base action and shares it with other agents, and the latter aggregates other agents' base actions to compute a perturbation. DeCOM then combines the base action and the perturbation to obtain the agent's final action. Such base action sharing mechanism provides an agent timely and necessary information about others, which helps the agent better regulate its actions for cooperation. 
	
	Furthermore, in our constrained MAS setting, agents' optimal policies correspond to the optimal solution of a constrained optimization which is intractable to solve directly. DeCOM addresses this issue by training the base policy to optimize the return and the perturbation policy to decrease the constraints violation, respectively. Such learning framework essentially decomposes the original constrained optimization into an  \textit{unconstrained optimization} and \textit{constraints satisfaction} problem, which is computationally efficient, easy-to-implement, and end-to-end.

	\textbf{Contributions.} The contributions of this paper are three-fold. First, to the best of our knowledge, this paper is the first that develops a constrained cooperative MARL framework. Our proposed framework DeCOM and its end-to-end training algorithm are both scalable and computationally efficient. Second, we give theoretical results which show that DeCOM's policy update algorithm is guaranteed to converge within limited number of steps under only mild assumptions. Third, in addition to toy environments, we also conduct experiments in a large-scale environment with 500 agents based on a real-world dataset with roughly 1 million ride-hailing orders from Nov. 1 to 30, 2016, in Chengdu, China. Furthermore, the various types of costs considered in the experiments, including unsafety, unfairness, and operational costs show the potentially wide applications of DeCOM.

%% file: 2.preliminary.tex
We consider \textit{constrained cooperative Markov game (CCMG)}, which is defined by a tuple $([N], \mathcal{S}, \{\mathcal{O}_i\}_{i=1}^N, \{\mathcal{A}_i\}_{i=1}^N, p, \{r_i\}_{i=1}^N, \{c^1_i\}_{i=1}^N, \cdots, \{c^M_i\}_{i=1}^N, \{D_j\}_{j=1}^M, p_0, \gamma)$. In a CCMG, $[N]=\{1,\cdots,N\}$ denotes the set of agents, and $\mathcal{S}$ denotes the global state space. Each agent $i$ has an observation space $\mathcal{O}_i$ and action space\footnote{We consider continuous action space in this paper.} $\mathcal{A}_i$. At each global state $s\in\mathcal{S}$, each agent $i$ only has limited observation $o_i=T(s, i)\in\mathcal{O}_i$, where $T(s, i)$ maps the global state $s$ to agent $i$'s observation. At each time step, each agent $i$ chooses an action $a_i$ from its action space $\mathcal{A}_i$. Given the joint action $\boldsymbol{a}=[a_1,\cdots,a_N]$ and the current global state $s$, the CCMG transits to the next global state $s'$ with probability $p(s'|s, \boldsymbol{a})$, and each agent $i$ receives an immediate reward $r_i(s, \boldsymbol{a})$ and $M$ types of immediate costs, denoted as $c^1_i(s, \boldsymbol{a}), c^2_i(s, \boldsymbol{a}),\cdots,c^M_i(s, \boldsymbol{a})$. Furthermore, $p_0$ denotes the initial global state distribution, and the constant $\gamma\in [0, 1]$ denotes the discount factor for future rewards.

Each agent $i$ selects its actions based on a local policy $\pi_i:\mathcal{O}_i\mapsto\Omega(\mathcal{A}_i)$, where $\Omega(\mathcal{A}_i)$ denotes all possible distributions over space $\mathcal{A}_i$. We denote $\boldsymbol{\pi}=[\pi_1,\cdots,\pi_N]\in\Psi$ as the joint policy of $N$ agents, where $\Psi$ denotes the set of all possible joint policies. Each agent $i$'s expected long-term discounted return $J^R_i(\boldsymbol{\pi})$ and expected long-term discounted cost $J^{C_j}_{i}(\boldsymbol{\pi})$ for each type $j\in[M]=\{1,\cdots,M\}$ are defined in Eq. (\ref{reward:expected_return}) and \eqref{cost:expected_return}, respectively.
\begin{align}
   J^R_i(\boldsymbol{\pi})&=\mathbb{E}_{\boldsymbol{\pi}, p, p_0}\bigg[\sum_{t=0}^{\infty}\gamma^t r_i(s_t, \boldsymbol{a_t})\bigg].\label{reward:expected_return}\\
   J^{C_j}_{i}(\boldsymbol{\pi})&=\mathbb{E}_{\boldsymbol{\pi}, p, p_0}\bigg[\sum_{t=0}^{\infty}\gamma^t c^{j}_i(s_t, \boldsymbol{a_t})\bigg].\label{cost:expected_return}
\end{align}

We consider the CCMG where agents work cooperatively to maximize the expected team-average return\footnote{Team-average reward has been widely considered in prior works (e.g., \cite{Qu2019ValuePF, pmlr-v80-zhang18n}).} $\frac{1}{N}\sum_{i\in[N]}J^R_i(\boldsymbol{\pi})$, and to  ensure that the expected team-average cost $\frac{1}{N}\sum_{i\in[N]}J^{C_j}_i(\boldsymbol{\pi})$ for each type $j\in[M]$ is upper bounded by $D_j$. Thus, the optimal joint policy to our CCMG is the optimal solution to the following Problem (\ref{original_problem}).

\begin{equation}\label{original_problem}
\begin{aligned}
    \underset{\boldsymbol{\pi}\in\Psi}{\max}\;J^R(\boldsymbol{\pi})& = \frac{1}{N}\sum_{i\in[N]}J^R_i(\boldsymbol{\pi}) \\
     \text{s.t. }J^{C_j}(\boldsymbol{\pi})& = \frac{1}{N}\sum_{i\in[N]}J^{C_j}_i(\boldsymbol{\pi}) \leq D_j,\;\; \forall{j}\in[M].
\end{aligned}    
\end{equation}

To obtain the optimal policy of our CCMG is exactly the objective of this paper. However, as we consider the practical scenario where the state transition kernel is unknown \textit{a priori}, the optimal policy could not be obtained by directly solving Problem (\ref{original_problem}). Thus, we take the approach of learning such policy via a novel framework of MARL, which will be elaborated in the following Section \ref{sec:framework}.


%% file: 3.framework.tex
\begin{figure}[t]
	\centering
	\subfigure[\textbf{Forward computation flow.}]{
		\begin{minipage}[t]{0.49\linewidth}
			\setlength{\abovecaptionskip}{0pt}
			\setlength{\belowcaptionskip}{5pt}
			\centering
			\includegraphics[width=2.7in]{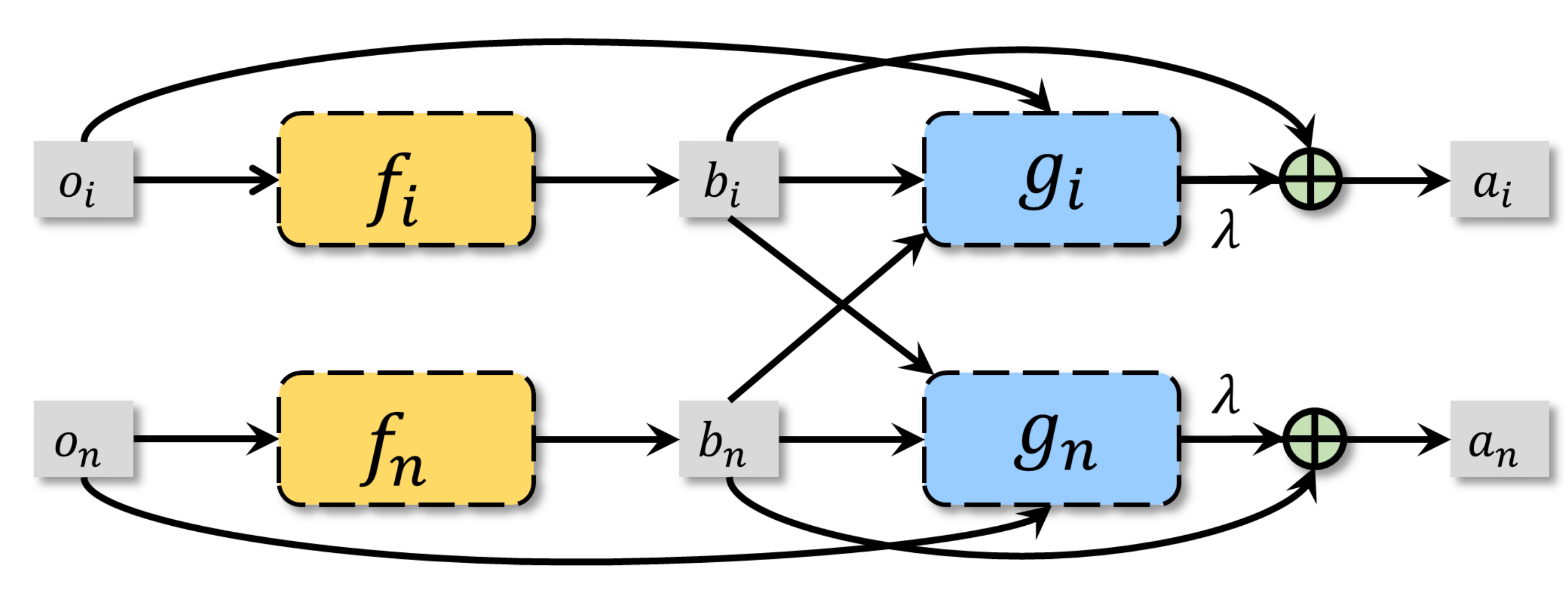}
			\label{fig:decom1}
	\end{minipage}}
	\subfigure[\textbf{Backward gradient flow.}]{
		\begin{minipage}[t]{0.49\linewidth}
			\setlength{\abovecaptionskip}{0pt}
			\setlength{\belowcaptionskip}{5pt}
			\centering
			\includegraphics[width=2.7in]{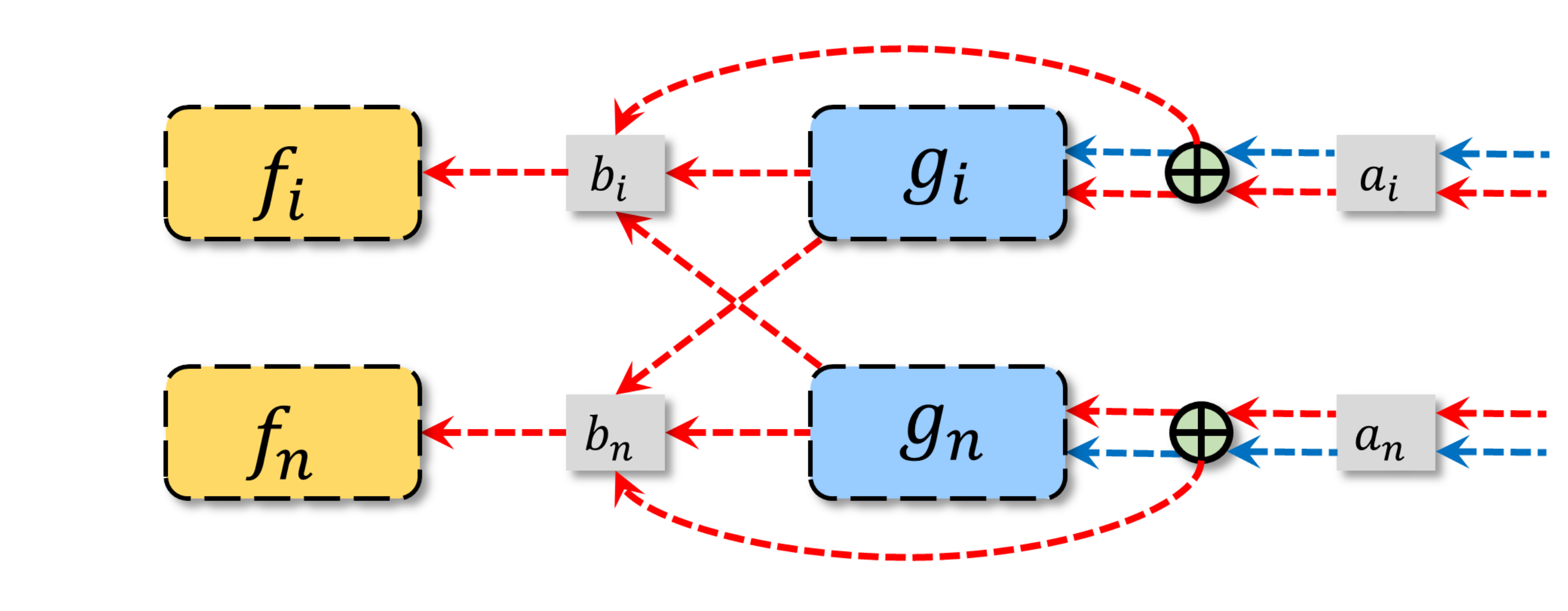}
			\label{fig:decom2}
	\end{minipage}}
	\vskip -5pt
	\caption{DeCOM framework for neighboring agents $i$ and $n$, where the red dashed arrows in Fig. \ref{fig:decom2} indicate the gradient to maximize the expected team-average return and the blue ones indicate the gradient to minimize the constraints violation.}
	\vskip -5pt
\end{figure}
As shown by Problem (\ref{original_problem}), in addition to cooperatively maximizing the expected team-average return, agents have to satisfy the constraint on each type of expected team-average cost. Such additional dimension of cooperation makes it more imperative that agents share timely and necessary information with others, so that agents could better regulate their actions based on their understandings about other agents. Therefore, we propose a novel MARL framework, named DeCOM, which enables communication among agents by decomposing the policy of each agent into a base policy and a perturbation policy, as shown in Figure \ref{fig:decom1}.

More specifically, at each time step, each agent $i$'s base policy $f_i$ receives a local observation $o_i$ from the environment, and outputs a base action $b_i\in\mathcal{A}_i$ which is shared with its neighbors. We define the neighbor set\footnote{The neighbor set can be decided by physical proximity or other factors, depending on the actual scenario.} of each agent $i$ as the set of agents that are able to communicate with it, and denote it as $\mathcal{N}_i$. Then, each agent $i$'s perturbation policy $g_i$ takes as inputs its observation $o_i$, its own base action $b_i$, as well as the base actions of its neighbors $\boldsymbol{b}_{\mathcal{N}_i}=[b_j]_{j\in\mathcal{N}_i}$, and outputs a perturbation $g_i(o_i, b_i, \boldsymbol{b}_{\mathcal{N}_i})$. A scaled perturbation is then added to $b_i$ to obtain the final action $a_i$. That is,
\begin{equation}\label{policy}
    \begin{aligned}
    a_i\;=\;b_i\;+\;\lambda\;g_i(o_i, b_i, \boldsymbol{b}_{\mathcal{N}_i}),\;
     \text{with }\; b_i\sim f_i(\cdot|o_i),
    \end{aligned}
\end{equation}
where $\lambda$ controls the magnitude of the perturbation, and $f_i(\cdot|o_i)$ is a probability distribution over the action space $\mathcal{A}_i$. Note that $f_i$ could also be a deterministic function, which is in fact a special case of a stochastic one. In contrast, DeCOM fixes $g_i$ as deterministic for strong representation power, as shown in the proof of Theorem \ref{representation_power} given in Appendix \ref{detail_proof:representation_power}.

Figure \ref{fig:decom2} also shows the gradient flows in the training procedure. In DeCOM, $f_i$ is updated by gradient ascent over $J^R(\boldsymbol{\pi})$, and thus $f_i$ is in charge of improving the expected team-average return. In contrast, $g_i$ receives the gradient to minimize constraints violation, making $g_i$ undertake the duty of perturbing the base action to satisfy constraints. Such modularization essentially divides Problem (\ref{original_problem}) into an unconstrained optimization problem and a constraints satisfaction problem. This ``divide and conquer'' method not only enables simple end-to-end training, but also avoids the heavy computation to solve complex constrained optimization problems which is inevitable in previous solution methods for constrained Markov decision process \cite{achiam2017constrained, pmlr-v119-satija20a, yu2019convergent, yang2020projection}. Furthermore, as shown in Figure \ref{fig:decom2}, DeCOM incorporates the gradients from $\mathcal{N}_i$ to update $f_i$, because gradient sharing among agents could facilitate agent cooperation as shown by recent studies \cite{DIAL, LAC}. 


We next show that DeCOM's decomposed policy structure dose not reduce the representation power in Theorem \ref{representation_power}, whose proof is given in Appendix \ref{detail_proof:representation_power}.

\begin{theorem}\label{representation_power}
Let $\Psi_{\textnormal{DeCOM}}$ contain all possible joint policies representable by DeCOM, and $\boldsymbol{\pi}^{*}\in\Psi$ be the optimal solution to Problem (\ref{original_problem}). Then, for the optimal joint policy $\boldsymbol{\pi}^+\in\Psi_{\textnormal{DeCOM}}$, we have $J^R(\boldsymbol{\pi}^+)=J^R(\boldsymbol{\pi}^{*})$ and $J^{C_j}(\boldsymbol{\pi}^+)=J^{C_j}(\boldsymbol{\pi}^{*})$, $\forall{j}\in[M]$.
\end{theorem}

Essentially, Theorem \ref{representation_power} states that the optimal joint policy under DeCOM yields the same expected term-average return and costs as that of the CCMG. Such result further validates our choice of decomposing the policy as in DeCOM. In this paper, we adopt the practical approach of realizing each $f_i$ and $g_i$ by neural networks, and denote the parameters for $f_i$ and $g_i$ as $\theta_i$ and $\phi_i$, respectively. To further simplify notation, we let $\boldsymbol{\theta}=[\theta_1,\cdots,\theta_N]$, $\boldsymbol{\phi}=[\phi_1,\cdots,\phi_N]$, and treat both $\boldsymbol{\theta}$ and $\boldsymbol{\phi}$ as vectors. We next represent agents' joint base policy as $\boldsymbol{f}=[f_1,\cdots,f_N]$,  and joint perturbation policy as $\boldsymbol{g}=[g_1,\cdots,g_N]$. Thus, under the DeCOM framework, the return and costs satisfy that $J^R(\boldsymbol{\pi})=J^R(\boldsymbol{f}, \boldsymbol{g})$ and $J^{C_j}(\boldsymbol{\pi})=J^{C_j}(\boldsymbol{f}, \boldsymbol{g}), \forall{j}\in[M]$. 

%% file: 4.training_alg.tex
\subsection{Algorithm Overview}

\begin{algorithm}[t]
	\small
	\noindent
    Initialize $\mathcal{D}\leftarrow\emptyset$; Initialize reward critic $Q^{\eta_0}$,
    cost critics $Q^{\zeta_{j,0}}, \forall{j}\in[M]$, $\boldsymbol{\theta}_0$, $\boldsymbol{\phi}_0$; Initialize $\lambda,\delta$\;\label{l1}  
    $\eta'_{0}\leftarrow \eta_{0}$,\,\,$\zeta_{j, 0}'\leftarrow \zeta_{j, 0}, \forall{j}\in[M]$,\,\,$\boldsymbol{\theta}'_{0}\leftarrow \boldsymbol{\theta}_{0}$,\,\,$\boldsymbol{\phi}'_{0}\leftarrow \boldsymbol{\phi}_{0}$\;\label{l2}  
    \ForEach{episode $k=0$ to max-episodes\label{l3}}{
        \ForEach{$t=0$ to episode-length\label{l4}}{
            Each agent $i$ selects base action $b_i$ based on $f_i$, and shares it with $\mathcal{N}_i$\;\label{l5}
            Each agent $i$ calculates action $a_i\leftarrow b_i+\lambda\,g_{i}(o_i, b_i, \boldsymbol{b}_{\mathcal{N}_i})$, and executes $a_i$\;\label{l6}
            Observe team-average reward $r$ and costs $c^j, \forall{j}\in[M]$, and next global state $s_{t+1}$\;\label{l7}
            Store experience $(s_t, \boldsymbol{b}, \boldsymbol{a}, r, \{c^j\}_{j=1}^M, s_{t+1})$ into $\mathcal{D}$\;\label{l8}
        }
        Sample a random mini-batch of $L$ transitions $\mathcal{B}=\{(s_l, \boldsymbol{b}_l, \boldsymbol{a}_l, r_l, \{c^j_l\}_{j=1}^M, s'_l)\}_{l=1}^L$ from $\mathcal{D}$\;\label{l9}
        Update reward critic by minimizing Eq. (\ref{td error:reward}), and cost critics by minimizing Eq. (\ref{td error:cost}) \;\label{l11}
        Update $\boldsymbol{\theta}_k$ and $\boldsymbol{\phi}_k$ to $\boldsymbol{\theta}_{k+1}$ and $\boldsymbol{\phi}_{k+1}$ according to Alg. \ref{updating_alg}\;\label{l13}
        $\eta'_{k+1}\leftarrow \delta\,\eta_{k+1}+(1-\delta)\,\eta'_{k}$,\,\,$\zeta_{j, k+1}'\leftarrow \delta\,\zeta_{j, k+1}+(1-\delta)\,\zeta_{j, k}', \forall{j}\in[M]$\;\label{l14}
        $\boldsymbol{\theta}'_{k+1}\leftarrow \delta\,\boldsymbol{\theta}_{k+1}+(1-\delta)\,\boldsymbol{\theta}'_{k}$,\,\,$\boldsymbol{\phi}'_{k+1}\leftarrow \delta\,\boldsymbol{\phi}_{k+1}+(1-\delta)\,\boldsymbol{\phi}'_{k}$\;\label{l15}
    }
    \caption{Training Algorithm of DeCOM}
    \label{decom_alg}
\end{algorithm}

Our training algorithm of DeCOM follows the actor-critic framework, as shown in Alg. \ref{decom_alg}. At each episode $k$, agents interact with the environment and the experiences of such interactions are collected into buffer $\mathcal{D}$ (line \ref{l4}-\ref{l8}). Then, the algorithm samples a mini-batch from $\mathcal{D}$, and updates the reward and cost critics by minimizing the TD error over the mini-batch (line \ref{l9}-\ref{l11}). After that, $\boldsymbol{\theta}$ and $\boldsymbol{\phi}$ get updated through Alg. \ref{updating_alg} (line \ref{l13}), which will be elaborated in Section \ref{sec:policy update}. Finally, Alg. \ref{decom_alg} performs soft update for the target networks to stabilize learning (line \ref{l14}-\ref{l15}). In what follows, we present our method of updating the critics, and the parameters $\boldsymbol{\theta}$ and $\boldsymbol{\phi}$ in detail.


\subsection{Updating Critics}
At each episode $k$ of Alg.\ref{decom_alg}, we update the reward critic by minimizing the TD error over the sampled mini-batch of $L$ transitions, given in the following Eq. \eqref{td error:reward}, 
\begin{align}\label{td error:reward}
    \mathcal{L}(\eta_k) = \frac{1}{L}\sum_{l=1}^L\big(q_l^R-Q^{\eta_k}(s_l, \boldsymbol{a}_l)\big)^2,\;\text{with }\; q_l^R=r_l+\gamma Q^{\eta'_k}(s'_l, \boldsymbol{a}'_l)|_{\boldsymbol{a}'_l\sim\boldsymbol{\pi}_{\boldsymbol{\theta}'_k,\boldsymbol{\phi}'_k}},
\end{align}
where $Q^{\eta_k}$ is the reward action-value function with parameter $\eta_k$, and $\boldsymbol{\pi}_{\boldsymbol{\theta}'_k,\boldsymbol{\phi}'_k}$ denotes the target joint policy with parameters $\boldsymbol{\theta}'_k$ and $\boldsymbol{\phi}'_k$. 

The cost critics are updated in a similar manner with the TD error given in Eq. \eqref{td error:cost} for each $j\in[M]$,
\begin{align}\label{td error:cost}
    \mathcal{L}(\zeta_{j,k}) = \frac{1}{L}\sum_{l=1}^L\big(q_l^{C_j}-Q^{\zeta_{j,k}}(s_l, \boldsymbol{a}_l)\big)^2,\;\text{with }\; q_l^{C_j}=c^j_l+\gamma Q^{\zeta'_{j,k}}(s'_l, \boldsymbol{a}'_l)|_{\boldsymbol{a}'_l\sim\boldsymbol{\pi}_{\boldsymbol{\theta}'_k,\boldsymbol{\phi}'_k}},
\end{align}
where $Q^{\zeta_{j,k}}$ is the action-value function on cost $j$ with parameter $\zeta_{j,k}$.


\subsection{Updating Policies}\label{sec:policy update}
\subsubsection{Algorithm Overview}
We present in Alg. \ref{updating_alg} the algorithm for updating the policy parameters, which is called on line \ref{l13} of Alg. \ref{decom_alg}. At each episode $k$ of Alg. \ref{decom_alg}, Alg. \ref{updating_alg} takes $\boldsymbol{\theta}_{k}$ and $\boldsymbol{\phi}_{k}$ as input, and updates $\boldsymbol{\theta}_k$ to $\boldsymbol{\theta}_{k+1}$ (line \ref{policy_update:l3}), and $\boldsymbol{\phi}_k$ to $\boldsymbol{\phi}_{k+1}$ (line \ref{policy_update:l4}-\ref{policy_update:l8}), whose details will be elaborated in Section \ref{optimize_f} and  \ref{optimize_g}.

\begin{algorithm}
\small
\textbf{Input:} Sampled mini-batch $\mathcal{B}$; $\boldsymbol{\theta}_k$, $\boldsymbol{\phi}_k$\;\label{policy_update:l1}
\textbf{Output:} $\boldsymbol{\theta}_{k+1}$, $\boldsymbol{\phi}_{k+1}$\;\label{policy_update:l2}
Update $\boldsymbol{\theta}_k$ to $\boldsymbol{\theta}_{k+1}$ by Eq. (\ref{update_f:stochastic_main_paper})\tcp*{\small By Eq.(\ref{update_f:deterministic}) in Appendix \ref{update_f}, if $\boldsymbol{f}$ is deterministic.}\label{policy_update:l3}
$\boldsymbol{\varphi}_0\leftarrow\boldsymbol{\phi}_k$\;\label{policy_update:l4}
\ForEach{$w=0$ to $W$\label{policy_update:l5}}{
    $j^*\leftarrow\arg\max_{j\in[M]}\Tilde{\mathcal{L}}_j(\boldsymbol{\varphi}_w; \boldsymbol{\theta}_{k+1})$\;\label{policy_update:l6}
    $\boldsymbol{\varphi}_{w+1}\leftarrow\Gamma_{\Phi}\Big[\boldsymbol{\varphi}_w-\tau\cdot\text{Clip}\big(\nabla_{\boldsymbol{\phi}}\mathcal{L}_{j^*}(\boldsymbol{\varphi}_w; \boldsymbol{\theta}_{k+1})\big)\Big]$\tcp*{$\Gamma_{\Phi}$ is the projection operator.}\label{policy_update:l7}
}
$\boldsymbol{\phi}_{k+1}\leftarrow\boldsymbol{\varphi}_W$\;\label{policy_update:l8}
\caption{Policy Update Algorithm}
\label{updating_alg}
\end{algorithm}

\subsubsection{Updating $\boldsymbol{\theta}$}\label{optimize_f}
DeCOM updates ${\boldsymbol{\theta}}$ via policy gradient methods. Next, we present the policy gradient of $J^R({\boldsymbol{f}}, \boldsymbol{g})$ under stochastic base policies in Theorem \ref{theo:pg}.

\begin{theorem}\label{theo:pg}
If each $f_i$ is stochastic (e.g., Gaussian policy), then at each episode $k$ of Alg. \ref{decom_alg}, the gradient of $J^R({\boldsymbol{f}}, \boldsymbol{g})$ w.r.t. $\theta_i, \forall{i}\in[N]$, is
\begin{align}\label{update_f:stochastic_main_paper}
    \nabla_{\theta_i}J^R({\boldsymbol{f}}, \boldsymbol{g}) \approx \mathbb{E}_{(s_0, \boldsymbol{b}, \boldsymbol{a})\sim \mathcal{B}}\Big[\nabla_{\theta_i}\log f_i(b_{i}|o_i)\;Q^{\eta_{k+1}}(s_0, \boldsymbol{a})+\nabla_{\theta_i}Q^{\eta_{k+1}}(s_0,  \boldsymbol{a})\Big].
\end{align}
\end{theorem}

Note that the term $\nabla_{\theta_i}Q^{\eta_{k+1}}(s_0,  \boldsymbol{a})$ in Eq. (\ref{update_f:stochastic_main_paper}) implicitly shows that agents do share gradients in the training process as discussed in Section \ref{sec:framework}. As DeCOM does not restrict $\boldsymbol{f}$ to be stochastic, we also derive the policy gradient under deterministic base policies, whose mathematical expression together with the proofs and more discussions on Theorem \ref{theo:pg} is presented in Appendix \ref{update_f}.



\subsubsection{Updating $\boldsymbol{\phi}$}\label{optimize_g}
In DeCOM, $\boldsymbol{g}$ perturbs the base action to satisfy constraints, whose parameter ${\boldsymbol{\phi}}$ solves the following constraints satisfaction problem:
\begin{equation}\label{constraints_satisfaction_prob}
\begin{aligned}
        \text{Find }\;\;\boldsymbol{\phi}\in\Phi,\;\;\text{s.t. }\;J^{C_j}(\boldsymbol{f}, \boldsymbol{g}) \leq D_j,\;\; \forall{j}\in[M],
\end{aligned}
\end{equation}
where $\Phi$ is the space of $\boldsymbol{\phi}$. As an exhaustive search for $\boldsymbol{\phi}$ that solves Problem (\ref{constraints_satisfaction_prob}) is intractable, we switch to a learning approach. Given any $\boldsymbol{\theta}$, for each $j\in[M]$, the constraint violation loss is defined as 
\begin{equation}\label{phi_loss:ideal}
\begin{aligned}
    \mathcal{L}_j(\boldsymbol{\phi}; \boldsymbol{\theta})=\Big(\max\big(0,\; J^{C_j}(\boldsymbol{f}, \boldsymbol{g}) - D_j\big)\Big)^2.
\end{aligned}  
\end{equation}

Given $\boldsymbol{\theta}_{k+1}$, we empirically approximate the above loss by
\begin{equation}\label{phi_loss:prac}
\begin{aligned}
    \Tilde{\mathcal{L}}_j(\boldsymbol{\phi}; \boldsymbol{\theta}_{k+1})=\Big(\max\big(0,\; \mathbb{E}_{(s_{0}, \boldsymbol{a})\sim\mathcal{B}}\big[Q^{\zeta_{j,k+1}}(s_{0}, \boldsymbol{a})\big] - D_j\big)\Big)^2.
\end{aligned}  
\end{equation}

As realized by the for loop (line \ref{policy_update:l5}-\ref{policy_update:l7}) in Alg. \ref{updating_alg}, we update $\boldsymbol{\phi}$ with $W$ iterations. This design is motivated by the convergence analysis presented in Section \ref{converge_analyze}, and practical settings of $W$ can be found in Appendix \ref{practical_implemetation}. In each iteration $w$ of Alg. \ref{updating_alg}, given the current value $\boldsymbol{\varphi}_w$ for the parameter $\boldsymbol{\phi}$, we find the cost $j^*\in[M]$ with the maximum empirical constraint violation loss $ \Tilde{\mathcal{L}}_j(\boldsymbol{\varphi}_w; \boldsymbol{\theta}_{k+1})$. Then, $\boldsymbol{\varphi}_w$ is updated to $\boldsymbol{\varphi}_{w+1}$ by projected gradient descent with the clipped version of the gradient $\nabla_{\boldsymbol{\phi}}\Tilde{\mathcal{L}}_{j^*}(\boldsymbol{\varphi}_w; \boldsymbol{\theta}_{k+1})$. That is,
\begin{equation}\label{varphi:empirical}
    \begin{aligned}
    \boldsymbol{\varphi}_{w+1}=\Gamma_{\Phi}\Big[\boldsymbol{\varphi}_w-\tau\cdot\text{Clip}\big(\nabla_{\boldsymbol{\phi}}\Tilde{\mathcal{L}}_{j^*}(\boldsymbol{\varphi}_w; \boldsymbol{\theta}_{k+1})\big)\Big],
 \end{aligned}
\end{equation}
where $\tau$ is the learning rate, $\Gamma_{\Phi}[\boldsymbol{\varphi}]$ projects $\boldsymbol{\varphi}$ into the space $\Phi$, and
\begin{equation}\label{clipped_grad}
\text{Clip}\big(\nabla_{\boldsymbol{\phi}}\Tilde{\mathcal{L}}_{j^*}(\boldsymbol{\varphi}_w; \boldsymbol{\theta}_{k+1})\big) =
\begin{cases}
  \nabla_{\boldsymbol{\phi}}\Tilde{\mathcal{L}}_{j^*}(\boldsymbol{\varphi}_w; \boldsymbol{\theta}_{k+1}), & \text{if}\;\;||\nabla_{\boldsymbol{\phi}}\Tilde{\mathcal{L}}_{j^*}(\boldsymbol{\varphi}_w; \boldsymbol{\theta}_{k+1})||\leq G{\color{red},}\\
  G\cdot\frac{\nabla_{\boldsymbol{\phi}}\Tilde{\mathcal{L}}_{j^*}(\boldsymbol{\varphi}_w; \boldsymbol{\theta}_{k+1})}{||\nabla_{\boldsymbol{\phi}}\Tilde{\mathcal{L}}_{j^*}(\boldsymbol{\varphi}_w; \boldsymbol{\theta}_{k+1})||}, & \text{otherwise},
\end{cases}
\end{equation}
with $G$ denoting the maximum allowable gradient norm. We adopt the above clipping operation to stabilize learning, which also helps Alg. \ref{updating_alg} converge, as shown in Section \ref{converge_analyze}. Furthermore, the estimation of constraint violation only considers the initial time step in Eq. (\ref{phi_loss:prac}) and (\ref{varphi:empirical}), in practice, we implement a more efficient estimation method which utilize backward value function \cite{pmlr-v119-satija20a} to assign the constraint violation to each time step. See Appendix \ref{practical_implemetation} for more details.

\subsubsection{Convergence Analysis}\label{converge_analyze}
Before formally stating Theorem \ref{main_theorem} on the convergence of Alg. \ref{updating_alg}, we introduce two mild assumptions\footnote{These assumptions are commonly adopted in existing works (e.g., \cite{tessler2018reward, ntk}).}, including that the space $\Phi$ is compact and convex, and that $\mathcal{L}_j(\boldsymbol{\phi}; \boldsymbol{\theta}_{k+1})$ is $L_j$-smooth 
	w.r.t. $\boldsymbol{\phi}$, $\forall{j}\in[M]$, with $L_{\max}$ denoting $\max\{L_1,\cdots,L_M\}$.

\begin{theorem}\label{main_theorem}
	Let $\boldsymbol{\phi}$ be updated with the exact constraint violation losses given $\boldsymbol{\theta}_{k+1}$. That is, in each iteration $w$ of Alg. \ref{updating_alg}, $\boldsymbol{\varphi}_{w+1}$ is set as $\Gamma_{\Phi}\big[\boldsymbol{\varphi}_w-\tau\cdot\textnormal{Clip}(\nabla_{\boldsymbol{\phi}}\mathcal{L}_{j^*}(\boldsymbol{\varphi}_w; \boldsymbol{\theta}_{k+1}))\big]$ with $j^*=\arg\max_{j\in[M]}\mathcal{L}_j(\boldsymbol{\varphi}_w; \boldsymbol{\theta}_{k+1})$. Then, for any $\epsilon>0$ and $j\in[M]$, if both $\tau L_{max}$ and $\tau G^2$ are sufficiently small, $\boldsymbol{\varphi}_w$ will converge in $H\leq \frac{\min_{\boldsymbol{\phi}\in\mathcal{X}}||\boldsymbol{\phi}_{k}-\boldsymbol{\phi}||^2}{2\tau\epsilon}$ steps to the region
	\begin{align*}
	C_k\leq\mathcal{L}_{j}(\boldsymbol{\phi}; \boldsymbol{\theta}_{k+1})\leq C_k+\frac{2\epsilon+\tau G^2}{2F(H)},
	\end{align*}
	where the set $\mathcal{X}=\arg\min_{\boldsymbol{\phi}\in\Phi}\mathcal{L}_{j^*}(\boldsymbol{\phi}; \boldsymbol{\theta}_{k+1})$, the value $C_k=\min_{\boldsymbol{\phi}\in\Phi}\mathcal{L}_{j^*}(\boldsymbol{\phi}; \boldsymbol{\theta}_{k+1})$, and $F(H)=\min\big(1,\frac{G}{||\nabla_{\boldsymbol{\phi}}\mathcal{L}_{j^*}(\boldsymbol{\varphi}_{{H}}; \boldsymbol{\theta}_{k+1})||}\big)$.
\end{theorem}

Theorem \ref{main_theorem} states that under mild conditions, Alg. \ref{updating_alg} converges within limited iterations. Specifically, if Problem (\ref{constraints_satisfaction_prob}) is feasible, which happens when constraint bounds $D_j$ are set appropriately, or the parameterization space $\Phi$ is large enough, then there exists $\boldsymbol{\phi}\in\Phi$ such that $\mathcal{L}_j(\boldsymbol{\phi}; \boldsymbol{\theta}_{k+1})=0, \forall{j}\in[M]$ and $C_k=0$. In this case, $\boldsymbol{\varphi}_w$ will converge to an approximately feasible solution of Problem (\ref{constraints_satisfaction_prob}) with a maximum constraint violation of $\sqrt{\frac{2\epsilon+\tau G^2}{2F(H)}}$. Theorem \ref{main_theorem} motivates us to set small $\tau$ and $G$, and use sufficient number of iterations to update $\boldsymbol{\varphi}_w$ in practice. See Appendix \ref{detail_proof:main_theorem} and \ref{practical_implemetation} for the proof of Theorem \ref{main_theorem} and detailed hyper-parameter settings, respectively.




%% file: 5.experiments.tex
\subsection{Simulation Environments and Costs}\label{env:intro}
To evaluate DeCOM, we construct four simulation environments, namely \textit{CTC-safe} and \textit{CTC-fair} which extend the cooperative treasure collection (CTC) environment \cite{MAAC}, as well as \textit{constrained directional sensor network (CDSN)} and \textit{constrained large-scale fleet management (CLFM)}.

\begin{figure}[H]
	\centering
	\subfigure{
		\begin{minipage}[t]{0.45\linewidth}
			\centering
			\includegraphics[width=2.2in]{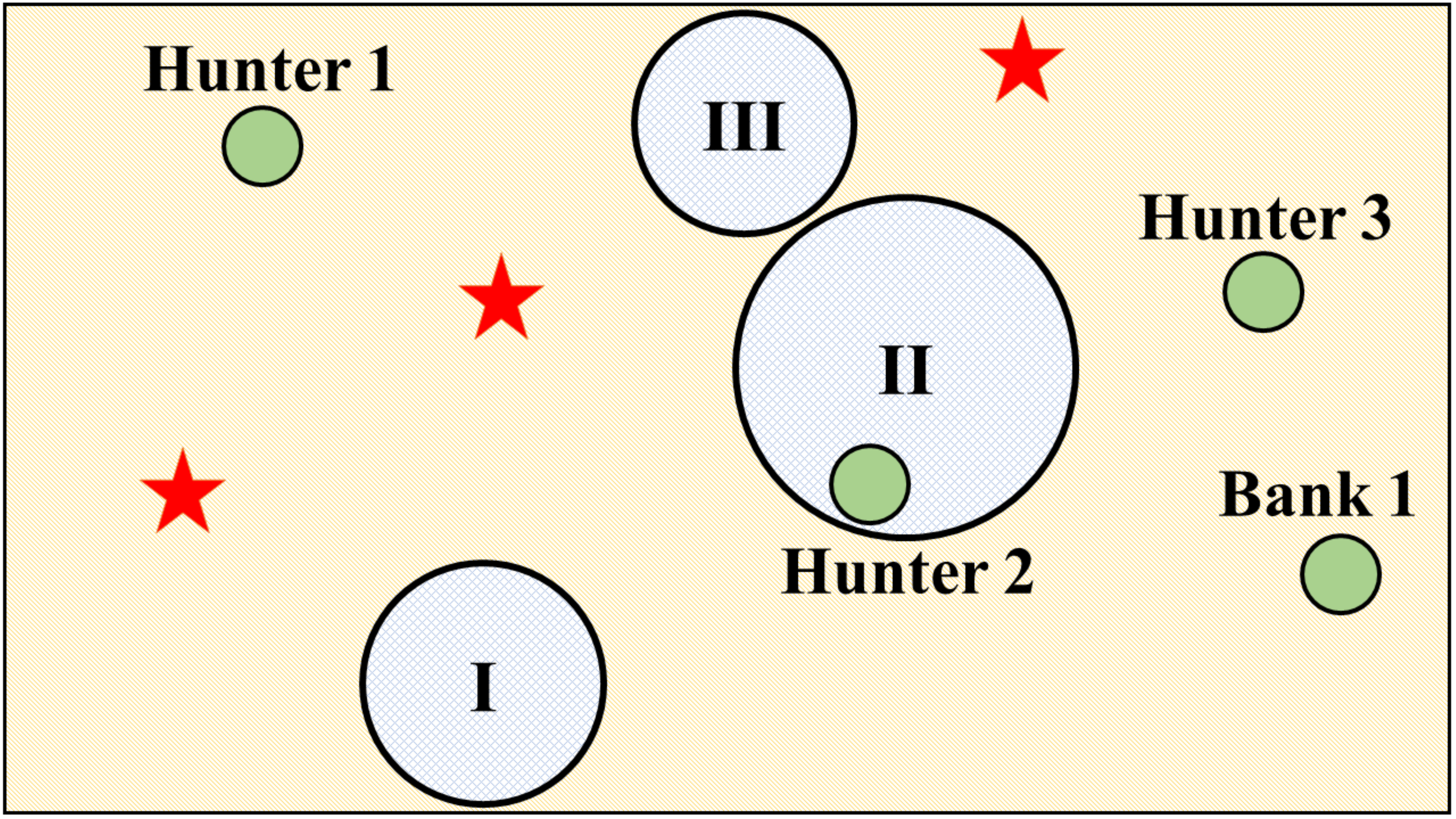}
			\caption*{CTC-safe.}
			\label{fig:ctc-safe}
	\end{minipage}}
	\subfigure{
		\begin{minipage}[t]{0.45\linewidth}
			\centering
			\includegraphics[width=2.2in]{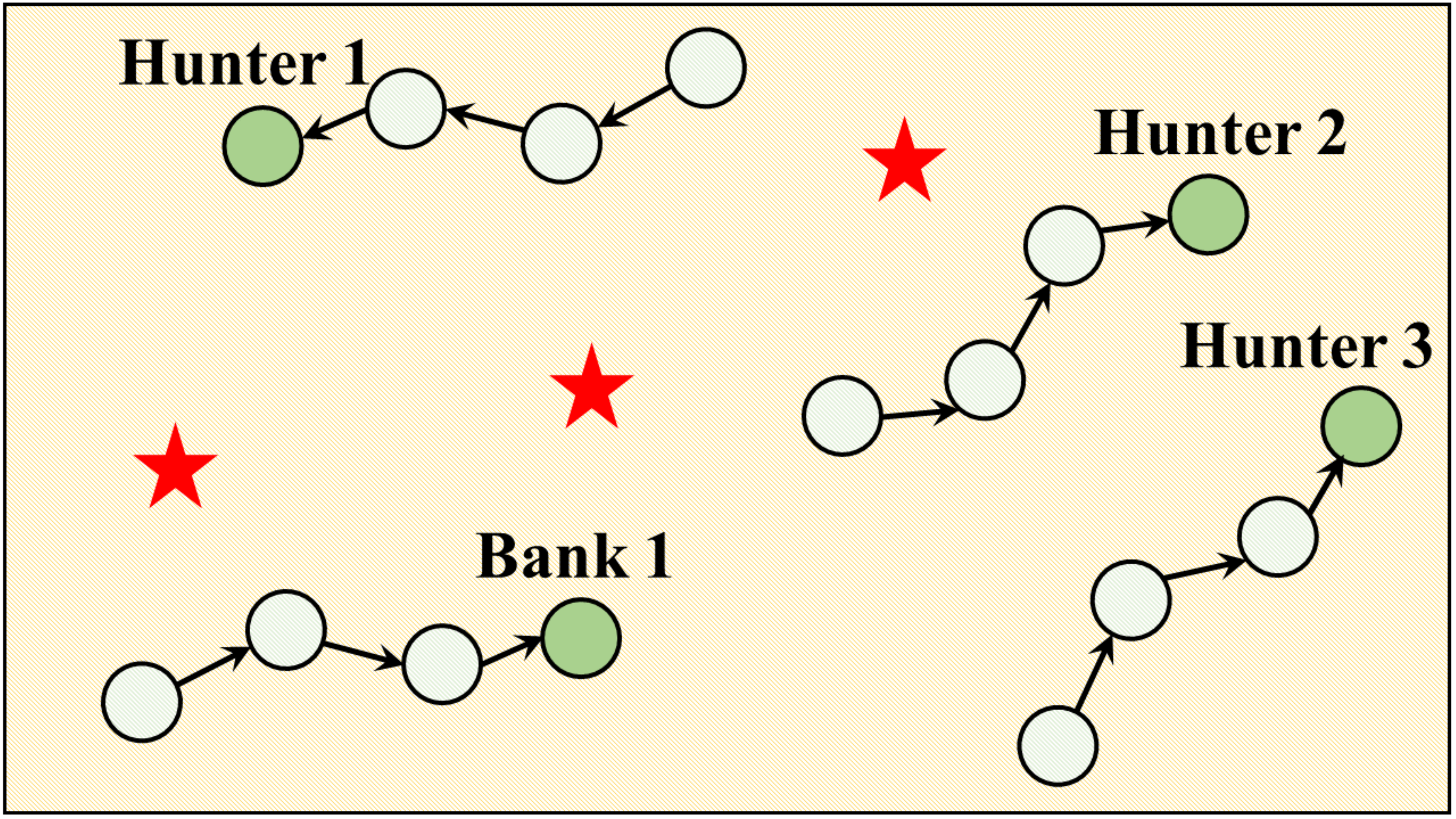}
			\caption*{CTC-fair.}
			\label{fig:ctc-fair}
	\end{minipage}}
	\subfigure{
		\begin{minipage}[t]{0.45\linewidth}
			\centering
			\includegraphics[width=2.2in]{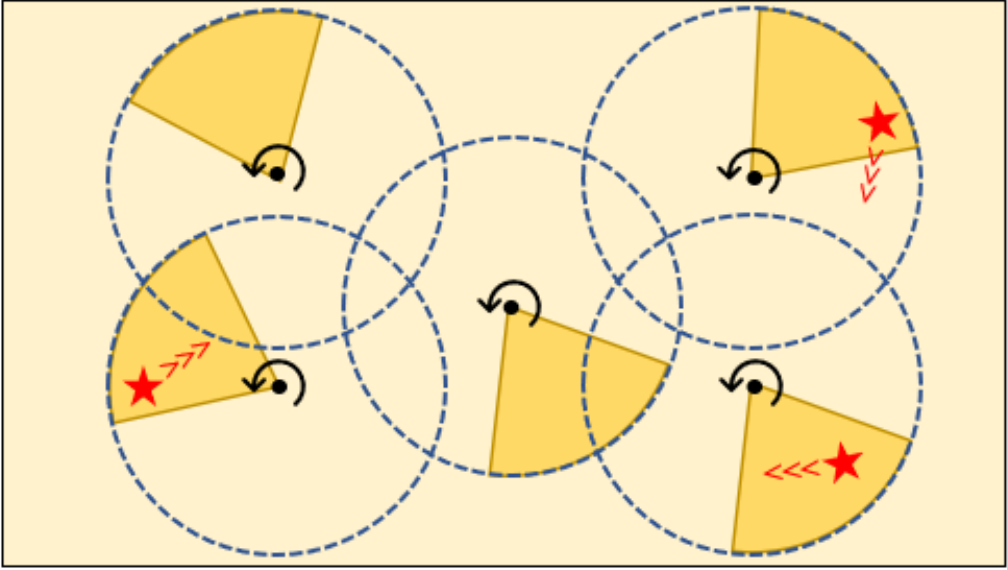}
			\caption*{CDSN.}
			\label{fig:cdsn}
		\end{minipage}}
	\subfigure{
		\begin{minipage}[t]{0.45\linewidth}
			\centering
			\includegraphics[width=2.2in]{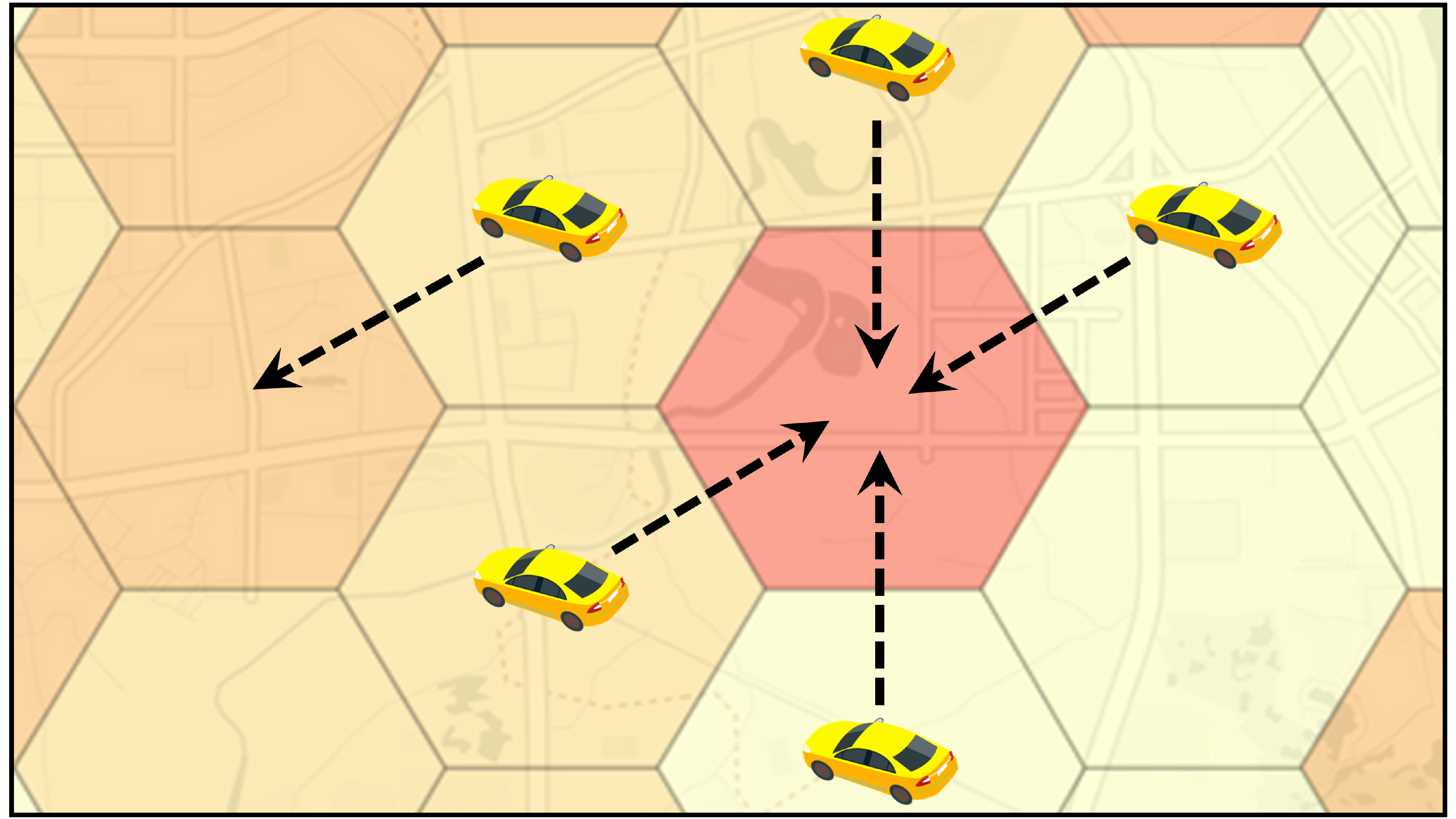}
			\caption*{CLFM.}
			\label{fig:clfm}
	\end{minipage}}
	\vskip -5pt
	\caption{Simulation environments, where stars represent treasures, areas I-III are unsafe regions 1-3 in Fig. \ref{fig:ctc-safe}, Fig. \ref{fig:ctc-fair} shows agents' trajectories. Black dots refer to deployed sensors and stars are moving objects to capture in Fig. \ref{fig:cdsn}. The dashed arrows represent repositioning decisions and darker grids have larger demand-supply gaps in Fig. \ref{fig:clfm}.}
	\vskip -2pt
\end{figure}


As in CTC, both CTC-safe and CTC-fair have two types of agents, namely \textit{hunters} and \textit{banks}. Hunters collect treasures and store them into banks, and treasures will get re-spawned randomly once collected. The action of an agent is to select a coordinate within a square box where it will reposition at the next time step. Each agent's reward is positively correlated with the amount of treasures stored in banks, and a hunter will be punished, if it collides with another hunter. Both CTC-safe and CTC-fair are instantiated with 3 hunters and 1 bank in our experiments.

CTC-safe adds $3$ randomly initialized unsafe regions into CTC, as shown in Fig. \ref{fig:ctc-safe}. Each unsafe region generates one type of cost, and each agent receives 1 for \textit{unsafety} cost $j$, if it locates in unsafe region $j$. Each agent in CTC-fair receives an \textit{unfairness} cost which equals to the maximal difference between agents' accumulated traveling distance.

CDSN extends the DSN environment \cite{xu2020learning} to continuous action space. In CDSN, sensors adjust their directions to capture moving objects, as shown in Fig. \ref{fig:cdsn}. Each agent receives two immediate reward: individual reward counts the number of objects it captured, shared global reward calculates the ratio of all captured objects. Furthermore, each agent also receives an operational cost positively related to the angle adjusted. The ultimate goal is to maxmize the accumulated combined reward, while satisfies the constraint over the average accumulated cost.

CLFM treats vehicles of an online-hailing platform\footnote{Examples include Didi Chuxing (http://www.didichuxing.com/en/) and Uber (https://www.uber.com/).} as agents, and focuses on relocating them in a distributed way, so as to maximize the revenue of the whole system under constraints on city-wide demand-supply gap and unfairness among drivers' incomes. More specifically, each agent receives a \textit{demand-supply gap} cost that equals to the KL-divergence between the idle vehicle and order distributions, as well as an \textit{unfairness} cost defined as the squared difference between agents' average accumulated income and its own. CLFM is built with a public city-scale dataset\footnote{Data source: DiDi Chuxing GAIA Open Dataset Initiative (https://gaia.didichuxing.com).} that contains approximate 1 million orders from November 1 to November 30, 2016 in Chengdu, China. In our simulation, we consider 500 vehicles and divide the urban area of Chengdu into 103 equal-size hexagon grids. An agent's action is to choose a weight vector locally, which is multiplied with the feature vector of each candidate grid to obtain a score. Then, the grid where the agent will reposition is sampled based on grids' scores. More detailed settings can be found in Appendix \ref{exp_details}.


\subsection{Algorithms and Neural Network Structures}
\begin{itemize}[leftmargin=*]
	\item \textbf{Fixed Penalty (FP)}. FP treats costs as penalties by adding to the reward. Each type of cost is multiplied with an identical weight chosen from the set $\mathcal{W}=\{0, -0.1, -1.0, -100\}$. We let FP-$|\omega|$ denote FP with weight $\omega\in\mathcal{W}$.     
	\item \textbf{Lagrangian (La)}. La extends RCPO \cite{tessler2018reward} to CCMG by replacing the single-agent reward (costs) with the team-average reward (costs) and single-agent policy with agents' joint policy.
	\item \textbf{Nocomm DeCOM (DeCOM-N)}. DeCOM-N is a variant of DeCOM where agents do not share base actions. That is, $a_i=b_i+\lambda\;g_{i}(o_i, b_i),\;\text{with }b_i\sim f_{i}(\cdot|o_i), \forall{i}\in[N]$.
	\item \textbf{Independent DeCOM (DeCOM-I)}. DeCOM-I is a variant of DeCOM, where the input of each $g_{i}$ is only $o_i$. That is, $a_i=b_i+\lambda\;g_{i}(o_i),\;\text{with }b_i\sim f_{i}(\cdot|o_i), \forall{i}\in[N]$.
	\item \textbf{DeCOM-A}. We denote the original DeCOM which retains all communication as DeCOM-A to align with its variants. 
\end{itemize}

Note that FP-0 is exactly the unconstrained MARL algorithm that aims to maximize the expected team-average return without any constraints. Based on the performance of FP-0, we select the neural network structures as follows: in CTC-safe and CTC-fair, we set $\boldsymbol{f}$ as deterministic and use MADDPG critics \cite{10.5555/3295222.3295385}; in CDSN, we use stochastic $\boldsymbol{f}$, and use MADDPG critics; in CLFM, we use stochastic $\boldsymbol{f}$, and Mean-Field critics \cite{mean_field}. We set $\lambda$ in CTC-safe, CDSN and CLFM to 1, and 0.01 in CTC-fair. See more discussion about chooing $\lambda$ in Appendix \ref{exp_details}, together with all training curves.

\subsection{Results Comparison}\label{exp:results_comparison}
Table \ref{ctc-table} shows the test results of CTC-safe and CTC-fair. In CTC-safe, the constraint bound for each unsafe region is set as $0.6$, $0.8$ and $1.0$. Among all algorithms, FP-0 achieves the highest reward, but violates all constraints.  Although DeCOM-A and its variants DeCOM-I and DeCOM-N do not have as high reward as FP-0, but they satisfy all constraints. La has the worst performance on all constraints.

In CTC-fair, the constraint bound for unfairness is 0. Test results in Table \ref{ctc-table} show that no algorithms satisfy the constraints. DeCOM-I and DeCOM-A have relatively low violation on unfairness, FP-0 and DeCOM-N have relatively high reward. Interestingly, DeCOM-A performs slightly better in CTC-fair comparing to CTC-safe. This phenomenon is highly related to the fact that ensuring fairness typically requires more agent interactions than safety. More specifically, an agent can avoid unsafe regions by its own observation, even without others' information. However, agent communication in DeCOM-A becomes more beneficial to help an agent control its traveling distance for fairness. However even with this promising observation, DeCOM-A still violates the constraint heavily. Achieving a better trade-off between reward and unfairness remains an interesting problem to explore in the future. 


\begin{table}
	\caption{CTC-safe and CTC-fair test results ($\text{average}\pm\text{standard deviation}$).}
	\label{ctc-table}
	\centering
	\resizebox{\textwidth}{!}{%
		\begin{tabular}{c|cccc|cc}
			\hline
			\multirow{2}{*}{\textbf{Algorithms}}   & \multicolumn{4}{c|}{\textbf{CTC-safe}}                                                                                      & \multicolumn{2}{c}{\textbf{CTC-fair}}                                                           \\ \cline{2-7}
			& \textbf{Reward} & \textbf{Unsafety 1} & \textbf{Unsafety 2} & \textbf{Unsafety 3} & \textbf{Reward} & \textbf{Unfairness} \\\hline
			\textbf{FP-0}   & 1.67   $\pm$ 1.08  & 0.82              $\pm$ 0.20 & 1.41              $\pm$ 0.20 & 2.02       $\pm$ 0.47 & 4.46	 $\pm$ 0.54	 & 8.91	  $\pm$ 1.03      \\
			\hline
			\textbf{FP-0.1} & -1.42           $\pm$ 0.86  & 0.75              $\pm$ 0.08 & 1.37              $\pm$ 0.13 & 2.16       $\pm$ 0.24  & 3.42	 $\pm$ 0.78	 & 9.96	  $\pm$ 1.65     \\
			\textbf{FP-1.0} & -1.74           $\pm$ 0.98  & 0.82              $\pm$ 0.61 & 1.46              $\pm$ 1.05 & 2.19       $\pm$ 1.44  & 0.00	 $\pm$ 0.25	 & 9.83	  $\pm$ 2.10     \\
			\textbf{FP-100} & -1.95           $\pm$ 0.95  & 0.61              $\pm$ 0.37 & 0.82           $\pm$ 0.65 & 1.45          $\pm$ 0.96   & -1.21	 $\pm$ 0.16	 & 11.99  $\pm$ 	4.38    \\
			\textbf{La}     & -1.52           $\pm$ 0.24  & 1.05              $\pm$ 0.35 & 1.47           $\pm$ 0.59 & 2.38          $\pm$ 0.69  & -1.41	 $\pm$ 0.08	 & 8.99	  $\pm$ 3.64     \\
			\textbf{DeCOM-I}     & \textbf{-0.74}           $\pm$ 0.45  & \textbf{0.38}     $\pm$ 0.09 & \textbf{0.44}  $\pm$ 0.27 & \textbf{0.65} $\pm$ 0.11  & 3.37	 $\pm$ 0.60	 & \textbf{8.05}	  $\pm$ 0.20     \\
			\textbf{DeCOM-N}     & -1.20           $\pm$ 0.06  & \textbf{0.43}     $\pm$ 0.28 & \textbf{0.41}  $\pm$ 0.26 & \textbf{0.75} $\pm$ 0.26   & \textbf{4.36}	 $\pm$ 0.65	 & 9.22	  $\pm$ 0.78    \\
			\textbf{DeCOM-A}  & -1.24           $\pm$ 0.25  & \textbf{0.43}     $\pm$ 0.11 & \textbf{0.45}  $\pm$ 0.19 & \textbf{0.64} $\pm$ 0.20   & 3.76	 $\pm$ 0.84	 & \textbf{8.11}	  $\pm$ 0.08    \\ \hline
		\end{tabular}%
	}
\end{table}

\begin{table}[]
	\caption{CDSN test results ($\text{average}\pm\text{standard deviation}$).}
	\label{cdsn_table}
	\centering
		\begin{tabular}{c|cccc}
			\hline
			\textbf{Algorithms}       & \textbf{Reward} & \textbf{Num. of Captured} & \textbf{Coverage Ratio} & \textbf{Operational Cost}  \\\hline
			\textbf{FP-0}   & 62.26  $\pm$ 4.73      & 22.23         $\pm$  1.59      & 0.40      $\pm$  0.03      & 69.65          $\pm$ 2.77      \\\hline
			\textbf{FP-0.1} & 50.96           $\pm$ 5.24      & 18.19                  $\pm$  1.56      & 0.32               $\pm$  0.04      & 61.10          $\pm$ 7.23      \\
			\textbf{FP-1.0} & 51.14           $\pm$ 7.32      & 18.24                  $\pm$  2.71      & 0.33               $\pm$  0.05      & 60.79          $\pm$ 6.35      \\
			\textbf{FP-100} & 47.72           $\pm$ 1.82      & 17.00                 $\pm$  0.60      & 0.30               $\pm$  0.01      & 63.49          $\pm$ 16.56     \\
			\textbf{Lang}   & 46.79           $\pm$ 0.47      & 16.71                  $\pm$  0.29      & 0.30               $\pm$  0.00      & 80.82          $\pm$ 23.01     \\
			\textbf{DeCOM-I}    & 34.67           $\pm$ 25.58     & 19.15                  $\pm$ 8.53      & 0.32               $\pm$  0.15     & 58.19          $\pm$ 15.41     \\
			\textbf{DeCOM-N}     & 51.93           $\pm$ 24.05     & 13.75                  $\pm$  9.79      & 0.21               $\pm$  0.16     & \textbf{32.88} $\pm$ 4.25      \\
			\textbf{DeCOM-A}  & \textbf{58.41}           $\pm$ 11.21     & \textbf{21.05}                  $\pm$  3.75      & \textbf{0.37}               $\pm$  0.10      & 58.90          $\pm$ 26.23    \\\hline
		\end{tabular}%
\end{table}

\begin{table}
	\caption{CLFM test results ($\text{average}\pm\text{standard deviation}$).}
	
	\label{clfm-table}
	\centering
	\begin{tabular}{c|cccc}
		\hline
		\textbf{Algorithms}   & \textbf{Revenue}   & \textbf{ORR}  & \textbf{Demand-Supply Gap} & \textbf{Unfairness}  \\\hline
		\textbf{FP-0}  & 18944.83          $\pm$ 48.91     & 0.49          $\pm$ 0.07      & 100.29         $\pm$ 1.80      & 74.57             $\pm$ 3.03      \\\hline
		\textbf{FP-0.1}  & 18809.87          $\pm$ 41.58     & 0.48          $\pm$ 0.03      & 103.25         $\pm$ 1.15      & 81.08             $\pm$ 3.31      \\
		\textbf{FP-1.0}  & 18815.76          $\pm$ 66.79     & 0.48          $\pm$ 0.08      & 103.10         $\pm$ 1.46      & 78.98             $\pm$ 2.13      \\
		\textbf{FP-100}  & 18840.14          $\pm$ 55.99     & 0.49          $\pm$ 0.10      & 100.52         $\pm$ 1.94      & 81.05             $\pm$ 3.78      \\
		\textbf{La}  & 18819.07          $\pm$ 235.92    & 0.49          $\pm$ 0.25      & 103.22         $\pm$ 4.19      & 82.54             $\pm$ 14.61     \\
		\textbf{DeCOM-I} & 19004.70          $\pm$ 76.88     & 0.49          $\pm$ 0.09      & 97.33          $\pm$ 1.28      & 70.40             $\pm$ 2.48      \\
		\textbf{DeCOM-N}     & 18668.26          $\pm$ 362.78    & 0.48          $\pm$ 0.37      & 105.57         $\pm$ 5.06      & 89.61             $\pm$ 20.29     \\
		\textbf{DeCOM-A}       & \textbf{19286.68} $\pm$ 78.42     & \textbf{0.50} $\pm$ 0.14      & \textbf{89.11} $\pm$ 3.53      & \textbf{62.99}    $\pm$ 1.35     \\\hline
	\end{tabular}
\end{table}

Table \ref{cdsn_table} lists the test results in CDSN. FP-0 and DeCOM-A achieve the highest reward, global coverage ratio and the number of captured objects. The constraint bound on the operational cost is set as 20. DeCOM-N has the lowest constraint violation. La has the worst constraint violation, as shown by its training curve in Fig. \ref{train:cdsn}.

Table \ref{clfm-table} shows the test results in CLFM, where the important metric \textit{Order Response Rate (ORR)} that measures the ratio of served orders is also given. The constraint bound for the demand-supply gap and unfairness cost is set as $90$ and $60$, respectively. As shown in Table \ref{clfm-table}, DeCOM-A satisfies both constraints. Specifically, DeCOM-A has the lowest demand-supply gap, which makes it reasonable for DeCOM-A to achieve the highest revenue and ORR. Meanwhile, DeCOM-A has the lowest constraint violation on unfairness. DeCOM’s better performance versus other baselines in CLFM comes from its communication mechanism, which essentially provides an agent the repositioning intentions of its neighbors. As we define an agent’s neighbors as those in its neighboring grids, such information could help an agent decide to reposition to grids with less vehicles and more orders.


%% file: 6.related_works.tex
\textbf{Multi-Agent Reinforcement Learning.} MARL is widely used to solve Markov games \cite{littman1994markov,zhang2020robust}, which can be categorized into competitive settings \cite{foerster2017learning,xie2020learning}, cooperative settings \cite{lin2019communication,bohmer2020deep}, and a mixture of them \cite{10.5555/3295222.3295385,bai2021sample}. As aforementioned, we focus on the cooperative setting in this paper. A series of recent MARL works for such settings, ranging from VDN \cite{sunehag2017value} to QMIX \cite{rashid2018qmix,rashid2020weighted}, adopt value-based methods that learn each agent's individual $Q$ function to represent the global $Q$ function by different mixing networks. Although these methods achieve good results for discrete action Markov games \cite{samvelyan2019starcraft}, they are generally not applied in our continuous action setting. DeCOM lies in the line of policy-based MARL methods, including MADDPG \cite{10.5555/3295222.3295385}, mean-field based method \cite{mean_field}, COMA \cite{Foerster2018CounterfactualMP}, MAAC \cite{MAAC}, and DOP \cite{DOP}. However, these methods are designed to solve unconstrained Markov games, which are thus not applicable in our constrained setting.

A family of MARL frameworks exploit \textit{communication} \cite{kim2018learning,ding2020learning,foerster2016learning,wang2020learning} by sharing either (encoded) observations or latent variables among agents. Different from them, the messages in DeCOM are agents' base actions, which have a clear and explicit meaning. Besides, DeCOM's base action sharing mechanism allows an agent to concurrently receive other agents' base actions and the gradient backflows at the current step, which brings more timeliness to help agents make better decisions.

\textbf{Constrained Reinforcement Learning.} A wide variety of constrained reinforcement learning frameworks are proposed to solve constrained MDPs (CMDPs) \cite{altman1999constrained}. They either convert a CMDP into an unconstrained min-max problem by introducing Lagrangian multipliers \cite{tessler2018reward, pmlr-v97-le19a, miryoosefi2019reinforcement, paternain2019constrained, turchetta2021safe, calian2020balancing,Chow2015RiskSensitiveAR}, or seek to obtain the optimal policy by directly solving constrained optimization problems \cite{achiam2017constrained, yang2020projection, zhang2020order, yu2019convergent, CEM, pmlr-v119-satija20a, ly1, ly2}. However, it is hard to scale these single-agent methods to our multi-agent setting due to computational inefficiency. 

Similar to DeCOM, one line of prior works \cite{luo2020multirobot,qin2021learning, Lu2020DecentralizedPG} also develop constrained MARL frameworks. However, \cite{luo2020multirobot,qin2021learning} focus on designing model-based control method to avoid collisions, which is not applicable to our scenario with an unknown state transition kernel; \cite{Lu2020DecentralizedPG} studies the scenario where each agent has a local safety constraint, whereas DeCOM is designed for cooperative settings with constraints on team-average costs. Furthermore, different from the aforementioned works \cite{luo2020multirobot,qin2021learning, Lu2020DecentralizedPG} that focus on specific applications, DeCOM could be applied in a wider range of applications that rely on intensive team cooperation, such as fleet management \cite{efficient_fm, coride}, order dispatch \cite{10.1145/3219819.3219824, li2019efficient}, multi-agent patrolling \cite{10.5555/3327757.3327905}, and target coverage in directional sensor networks \cite{xu2020learning}.


%% file: appendix_proof.tex
\subsection{Proof of Theorem \ref{representation_power}}\label{detail_proof:representation_power}
We restate the theorem \ref{representation_power} as follows.
\begin{theorem*}
Let $\Psi_{\textnormal{DeCOM}}$ contain all possible joint policies representable by DeCOM, and $\boldsymbol{\pi}^{*}\in\Psi$ be the optimal solution to Problem (\ref{original_problem}). Then, for the optimal joint policy $\boldsymbol{\pi}^+\in\Psi_{\textnormal{DeCOM}}$, we have $J^R(\boldsymbol{\pi}^+)=J^R(\boldsymbol{\pi}^{*})$ and $J^{C_j}(\boldsymbol{\pi}^+)=J^{C_j}(\boldsymbol{\pi}^{*})$, $\forall{j}\in[M]$.
\end{theorem*}
\begin{proof}
We firstly establish the following useful lemma.
\begin{lemma}\label{optimal_rep}
Given any policy $\boldsymbol{\pi}\in\Psi$, there exist $f$ and $\boldsymbol{g}$ that achieves the equivalent average long term reward as $\boldsymbol{\pi}$, formally
\begin{align*}
    J^R(\boldsymbol{\pi})=\int_{s_0}p_0(s_0)\int_{\boldsymbol{b}}f(\,\boldsymbol{b}\,|\,s_0\,)\;Q^{\boldsymbol{\pi}}\big(s_0, \boldsymbol{b}+\lambda\;\boldsymbol{g}(s_0, \boldsymbol{b})\big)\;d\boldsymbol{b}\;ds_0,
\end{align*}
where $Q^{\boldsymbol{\pi}}$ represents the action value function.
\end{lemma}

\begin{proof}
Since we are considering the general policy $\boldsymbol{\pi}\in\Psi$, we assume it is stochastic. We can have following derivations:
\begin{align*}
    J^R(\boldsymbol{\pi}) & = \int_{s_0}p_0(s_0)\int_{\boldsymbol{a}}\boldsymbol{\pi}(\boldsymbol{a}|s_0)\;Q^{\boldsymbol{\pi}}(s_0, \boldsymbol{a})\;d\boldsymbol{a}\;ds_0 \\
    & = \int_{s_0}p_0(s_0)\int_{\boldsymbol{a}}\boldsymbol{\pi}(\boldsymbol{a}|s_0)\int_{s'}p(s'|s_0, \boldsymbol{a})\;G\;ds'\;d\boldsymbol{a}\;ds_0
\end{align*}
where $G=r(s_0, \boldsymbol{a})+\gamma V^R_{\boldsymbol{\pi}}(s')$.
\begin{align*}
    J^R(\boldsymbol{\pi}) & = \int_{s_0}p_0(s_0)\int_{\boldsymbol{a}}\int_{s'}\boldsymbol{\pi}(\boldsymbol{a}|s_0)\frac{P(s', s_0, \boldsymbol{a})}{P(s_0, \boldsymbol{a})}\;G\;ds'\;d\boldsymbol{a}\;ds_0 \\
    & = \int_{s_0}p_0(s_0)\int_{\boldsymbol{a}}\int_{s'}\frac{P(s', s_0, \boldsymbol{a})}{P(s_0)}\;G\;ds'\;d\boldsymbol{a}\;ds_0 \\
    & = \int_{s_0}p_0(s_0)\int_{\boldsymbol{a}}\int_{s'}\frac{P(\boldsymbol{a}|s', s_0)P(s_0, s')}{P(s_0)}\;G\;ds'\;d\boldsymbol{a}\;ds_0 \\
    & = \int_{s_0}p_0(s_0)\int_{\boldsymbol{a}}\int_{s'}\int_{\boldsymbol{b}}\frac{P(\boldsymbol{a}, \boldsymbol{b}|s', s_0)P(s_0, s')}{P(s_0)}\;G\;d\boldsymbol{b}\;ds'\;d\boldsymbol{a}\;ds_0 \\
    & = \int_{s_0}p_0(s_0)\int_{\boldsymbol{a}}\int_{s'}\int_{\boldsymbol{b}}P(\boldsymbol{b}|s_0)\frac{P(\boldsymbol{a}, \boldsymbol{b}|s', s_0)P(s_0, s')}{P(\boldsymbol{b}|s_0)P(s_0)}\;G\;d\boldsymbol{b}\;ds'\;d\boldsymbol{a}\;ds_0 \\
    & = \int_{s_0}p_0(s_0)\int_{\boldsymbol{a}}\int_{\boldsymbol{b}}P(\boldsymbol{b}|s_0)\int_{s'}\frac{P(\boldsymbol{a}, \boldsymbol{b}, s', s_0)}{P(\boldsymbol{b}, s_0)}\;G\;ds'\;d\boldsymbol{b}\;d\boldsymbol{a}\;ds_0 \\
    & = \int_{s_0}p_0(s_0)\int_{\boldsymbol{a}}\int_{\boldsymbol{b}}P(\boldsymbol{b}|s_0)\int_{s'}P(\boldsymbol{a}, s'|\boldsymbol{b}, s_0)\;G\;ds'\;d\boldsymbol{b}\;d\boldsymbol{a}\;ds_0 \\
    & = \int_{s_0}p_0(s_0)\int_{\boldsymbol{a}}\int_{\boldsymbol{b}}P(\boldsymbol{b}|s_0)\int_{s'}P(s'|\boldsymbol{a}, \boldsymbol{b}, s_0)P(\boldsymbol{a}|\boldsymbol{b}, s_0)\;G\;ds'\;d\boldsymbol{b}\;d\boldsymbol{a}\;ds_0 \\
    & \stackrel{*}{=} \int_{s_0}p_0(s_0)\int_{\boldsymbol{a}}\int_{\boldsymbol{b}}P(\boldsymbol{b}|s_0)\int_{s'}p(s'|\boldsymbol{a}, s_0)P(\boldsymbol{a}|\boldsymbol{b}, s_0)\;G\;ds'\;d\boldsymbol{b}\;d\boldsymbol{a}\;ds_0\\
    & \stackrel{**}{=} \int_{s_0}p_0(s_0)\int_{\boldsymbol{a}:\,\boldsymbol{a}=\boldsymbol{b}+\lambda\;\boldsymbol{g}(s_0, \boldsymbol{b})}\int_{\boldsymbol{b}}P(\boldsymbol{b}|s_0)\int_{s'}p(s'|\boldsymbol{a}, s_0)\;G\;ds'\;d\boldsymbol{b}\;d\boldsymbol{a}\;ds_0 \\
    & \stackrel{***}{=} \int_{s_0}p_0(s_0)\int_{\boldsymbol{b}}f(\,\boldsymbol{b}\,|\,s_0\,)\;Q^{\boldsymbol{\pi}}\big(s_0, \boldsymbol{b}+\lambda\;\boldsymbol{g}(s_0, \boldsymbol{b})\big)\;d\boldsymbol{b}\;ds_0, \\ 
\end{align*}
In the derivation, $*$ is because given $s_0$ and $\boldsymbol{a}$, the probability of next state $s'$ can be quantified by the transition kernel $p$, which has nothing to do with $\boldsymbol{b}$. $**$ is because given $s_0$ and $\boldsymbol{b}$, then $\boldsymbol{a}$ is definite since function $\boldsymbol{g}$ is deterministic. $***$ removes the integral on $\boldsymbol{a}$, replaces $P(\,\boldsymbol{b}\,|\,s_0\,)$ by $f(\,\boldsymbol{b}\,|\,s_0\,)$ and absorbs the integral on $s'$ into the action value function $Q^{\boldsymbol{\pi}}$.
\end{proof}

Define the feasible set as $FS=\{\boldsymbol{\pi}\in\Psi\,\big|\,J^{C_j}(\boldsymbol{\pi})\leq D_j, \forall{j}\in[M]\}$, then the optimal policy is chosen by:
\begin{align*}
    \boldsymbol{\pi}^{*} = \arg\,\max_{\boldsymbol{\pi}\in FS} J^R(\boldsymbol{\pi}).
\end{align*}
Recall the proof of lemma (\ref{optimal_rep}), we can make further deduction, as showed in the following:
\begin{equation}\label{occupancy_measure}
\begin{aligned}
        J^R(\boldsymbol{\pi}) & = \int_{s_0}p_0(s_0)\int_{\boldsymbol{b}}f(\,\boldsymbol{b}\,|\,s_0\,)\;Q^{\boldsymbol{\pi}}\big(s_0, \boldsymbol{b}+\lambda\;\boldsymbol{g}(s_0, \boldsymbol{b})\big)\;d\boldsymbol{b}\;ds_0 \\
        & =\mathbb{E}_{(s, \boldsymbol{a})\sim\Lambda\big(p_0, f, \boldsymbol{g}, p, \lambda, \gamma\big)}\,\Big[r(s, \boldsymbol{a})\Big],
\end{aligned}
\end{equation}
which holds for any $\boldsymbol{\pi}\in\Psi$. Notice that in equation (\ref{occupancy_measure}), the expectation is over $\Lambda(\cdot)$, the distribution of state $s$ and joint action $\boldsymbol{a}$, which is also know as the \textit{cumulative discounted state-action occupancy measure} \cite{Zhang2020VariationalPG} . Further, equation (\ref{occupancy_measure}) implies that for any $\boldsymbol{\pi}\in\Psi$, there exists $f$ and $\boldsymbol{g}$ that generate the same cumulative discounted state-action occupancy measure as itself. Thus by choosing $r(s, \boldsymbol{a})=c^j(s, \boldsymbol{a})$, we can obtain $J^{C_j}(\boldsymbol{\pi}), \forall{j}\in[M]$. For the optimal policy $\boldsymbol{\pi}^{*}$, there also exists $f$ and $\boldsymbol{g}$ such that for the joint policy $\boldsymbol{\pi}^{+}\in\Psi_{DeCOM}$, $J^R(\boldsymbol{\pi}^+)=J^R(\boldsymbol{\pi}^{*})$ and $J^{C_j}(\boldsymbol{\pi}^+)=J^{C_j}(\boldsymbol{\pi}^{*})$, $\forall{j}\in[M]$. 

The optimality of $\boldsymbol{\pi}^{+}$ can be derived via contradiction. If $\boldsymbol{\pi}'\in\Psi_{DeCOM}$ has value $J^R(\boldsymbol{\pi}')>J^R(\boldsymbol{\pi}^{+})$ and $J^{C_j}(\boldsymbol{\pi}')=J^{C_j}(\boldsymbol{\pi}^{+})$, $\forall{j}\in[M]$, then $\boldsymbol{\pi}^{*}$ is not optimal for Problem \ref{original_problem}. Thus we complete the proofs.
\end{proof}

\subsection{Proof of Theorem \ref{theo:pg}}\label{update_f}
We present the policy gradient of $J^R(\boldsymbol{f}, \boldsymbol{g})$ under stochastic and deterministic base policies in following theorem.
\begin{theorem*}
If each $f_i$ is stochastic (e.g., Gaussian policy), then at each episode $k$ of Alg. \ref{decom_alg}, the gradient of $J^R({\boldsymbol{f}}, \boldsymbol{g})$ w.r.t. $\theta_i, \forall{i}\in[N]$, is
\begin{align}\label{update_f:stochastic_appendix}
	\nabla_{\theta_i}J^R({\boldsymbol{f}}, \boldsymbol{g}) \approx \mathbb{E}_{(s_0, \boldsymbol{b}, \boldsymbol{a})\sim \mathcal{B}}\Big[\nabla_{\theta_i}\log f_i(b_{i}|o_i)\;Q^{\eta_{k+1}}(s_0, \boldsymbol{a})+\nabla_{\theta_i}Q^{\eta_{k+1}}(s_0,  \boldsymbol{a})\Big].
\end{align}
If each $f_i$ is deterministic, then at each episode $k$ of Alg. \ref{decom_alg}, the gradient of $J^R({\boldsymbol{f}}, \boldsymbol{g})$ w.r.t. $\theta_i, \forall{i}\in[N]$, is
\begin{align}\label{update_f:deterministic}
    \nabla_{\theta_i}J^R({\boldsymbol{f}}, \boldsymbol{g}) \approx \mathbb{E}_{(s, \boldsymbol{a})\sim\mathcal{B}}\bigg[\;\nabla_{\theta_i}\boldsymbol{\pi}_{\boldsymbol{\theta},\boldsymbol{\phi}}(s)\;\nabla_{\boldsymbol{a}}\;Q^{\eta_{k+1}}\big(s, \boldsymbol{a}\big)\,\big|_{\boldsymbol{a}=\boldsymbol{\pi}_{\boldsymbol{\theta},\boldsymbol{\phi}}(s)}\;\bigg].
\end{align}
\end{theorem*}
\begin{proof}
	We firstly consider stochastic $f_i$. Using the stochastic policy gradient theorem \cite{sutton2018reinforcement} on $\boldsymbol{\theta}$, we have that
	\begin{align}
	\nabla_{\boldsymbol{\theta}}J^R(\boldsymbol{f}, \boldsymbol{g}) & = \nabla_{\boldsymbol{\theta}} \int_{s_0}p_0(s_0)\int_{\boldsymbol{a}}\boldsymbol{\pi}_{\boldsymbol{\theta},\boldsymbol{\phi}}(\boldsymbol{a}|s_0)\;Q^{\boldsymbol{\pi}_{\boldsymbol{\theta},\boldsymbol{\phi}}}(s_0, \boldsymbol{a})\,d\boldsymbol{a}\;ds_0\label{whole_pg: 1}\\ 
	& = \int_{s}\mu(s)\int_{\boldsymbol{a}}\nabla_{\boldsymbol{\theta}}\boldsymbol{\pi}_{\boldsymbol{\theta},\boldsymbol{\phi}}(\boldsymbol{a}|s)\;Q^{\boldsymbol{\pi}_{\boldsymbol{\theta},\boldsymbol{\phi}}}(s, \boldsymbol{a})\,d\boldsymbol{a}\;ds\label{whole_pg: 2}\\
	& = \int_{s}\mu(s)\int_{\boldsymbol{a}}\int_{\boldsymbol{b}}\nabla_{\boldsymbol{\theta}}f(\boldsymbol{b}|s)\;Q^{\boldsymbol{\pi}_{\boldsymbol{\theta},\boldsymbol{\phi}}}(s, \boldsymbol{b}+\boldsymbol{g}(s, \boldsymbol{b}))\;\mathcal{I}\big(\boldsymbol{a}=\boldsymbol{b}+\lambda\;\boldsymbol{g}(s, \boldsymbol{b})\big)\,d\boldsymbol{b}\,d\boldsymbol{a}\;ds, \label{whole_pg: 3},
	\end{align}
	where $\mu(s), s\in\mathcal{S}$ represents the on-policy distribution \cite{sutton2018reinforcement} under policy $\boldsymbol{\pi}_{\boldsymbol{\theta},\boldsymbol{\phi}}$, $Q^{\boldsymbol{\pi}_{\boldsymbol{\theta},\boldsymbol{\phi}}}$ represents  action-value function, $\mathcal{I}(\cdot)$ represents a indicator function, which equals $1$ if the inner condition is satisfied otherwise $0$. Eq. (\ref{whole_pg: 2}), is derived based on the policy gradient theorem \cite{sutton2018reinforcement}, Eq. (\ref{whole_pg: 3}) is derived based on the DeCOM framework. However in Eq. (\ref{whole_pg: 3}), for each joint action $\boldsymbol{a}$, the inner integral has to search all $\boldsymbol{b}$ that satisfies $\boldsymbol{a}=\boldsymbol{b}+\lambda\;\boldsymbol{g}(s, \boldsymbol{b})$, which can be intractable and biased if it is only approximated via limited sampled experiences in practice. To stay tractable, we choose to update $\boldsymbol{\theta}$ using Eq. (\ref{whole_pg: 1}) directly, which can have further derivations:
	\begin{align}
	\nabla_{\boldsymbol{\theta}} J^R(\boldsymbol{f},\boldsymbol{g}) & = \nabla_{\boldsymbol{\theta}} \int_{s_0}p_0(s_0)\int_{\boldsymbol{a}}\boldsymbol{\pi}_{\boldsymbol{\theta},\boldsymbol{\phi}}(\boldsymbol{a}|s_0)\;Q^{\boldsymbol{\pi}_{\boldsymbol{\theta},\boldsymbol{\phi}}}(s_0, \boldsymbol{a})\,d\boldsymbol{a}\;ds_0\nonumber\\
	& = \nabla_{\boldsymbol{\theta}}\int_{s_0}p_0(s_0)\int_{\boldsymbol{b}}\boldsymbol{f}(\,\boldsymbol{b}\,|\,s_0\,)\;Q^{\boldsymbol{\pi}_{\boldsymbol{\theta},\boldsymbol{\phi}}}\big(s_0, \boldsymbol{b}+\lambda\;\boldsymbol{g}(s_0, \boldsymbol{b})\big)\;d\boldsymbol{b}\;ds_0\label{used_pg: 2}\\
	& = \int_{s_0}p_0(s_0)\int_{\boldsymbol{b}}\nabla_{\boldsymbol{\theta}}\boldsymbol{f}(\,\boldsymbol{b}\,|\,s_0\,)\;Q^{\boldsymbol{\pi}_{\boldsymbol{\theta},\boldsymbol{\phi}}}\big(s_0, \boldsymbol{b}+\lambda\;\boldsymbol{g}(s_0, \boldsymbol{b})\big)+\nonumber\\
	&\;\;\;\;\boldsymbol{f}(\,\boldsymbol{b}\,|\,s_0\,)\nabla_{\boldsymbol{\theta}}Q^{\boldsymbol{\pi}_{\boldsymbol{\theta},\boldsymbol{\phi}}}\big(s_0, \boldsymbol{b}+\lambda\;\boldsymbol{g}(s_0, \boldsymbol{b})\big)\;d\boldsymbol{b}\;ds_0\label{used_pg: 3}\\
	& \approx \mathbb{E}_{(s_0, \boldsymbol{b}, \boldsymbol{a})\sim \mathcal{B}}\bigg[\nabla_{\boldsymbol{\theta}}\log \boldsymbol{f}(\boldsymbol{b}|s_0)\;Q^{\eta_{k+1}}(s_0, \boldsymbol{a})+\nabla_{\boldsymbol{\theta}}Q^{\eta_{k+1}}(s_0,  \boldsymbol{a})\bigg].\label{used_pg: 4}
	\end{align}
	Eq. (\ref{used_pg: 2}) is based on the results of Lemma \ref{optimal_rep}, Eq. (\ref{used_pg: 3}) is based on the product rule when differentiation, Eq. (\ref{used_pg: 4}) is the practical usage with experience buffer $\mathcal{B}$. Thus the gradient of $J^R(\boldsymbol{f},\boldsymbol{g})$ w.r.t. $\theta_i$ is
	\begin{align*}
	\nabla_{\theta_i}J^R({\boldsymbol{f}}, \boldsymbol{g}) \approx \mathbb{E}_{(s_0, \boldsymbol{b}, \boldsymbol{a})\sim \mathcal{B}}\Big[\nabla_{\theta_i}\log f_i(b_{i}|o_i)\;Q^{\eta_{k+1}}(s_0, \boldsymbol{a})+\nabla_{\theta_i}Q^{\eta_{k+1}}(s_0,  \boldsymbol{a})\Big].
	\end{align*}
	When each $f_i$ is deterministic, then the joint policy $\boldsymbol{\pi}$ is also deterministic. Following deterministic policy gradient theorem \cite{silver2014deterministic} we have
	\begin{align}
	\nabla_{\boldsymbol{\theta}} J^R(\boldsymbol{f},\boldsymbol{g}) & = \nabla_{\boldsymbol{\theta}}  \int_{s}\mu(s)\;Q^{\boldsymbol{\pi}_{\boldsymbol{\theta},\boldsymbol{\phi}}}\big(s, \boldsymbol{\pi}_{\boldsymbol{\theta},\boldsymbol{\phi}}(s)\big)\;ds\nonumber\\
	& = \int_{s}\mu(s)\;\nabla_{\boldsymbol{\theta}} Q^{\boldsymbol{\pi}_{\boldsymbol{\theta},\boldsymbol{\phi}}}\big(s, \boldsymbol{\pi}_{\boldsymbol{\theta},\boldsymbol{\phi}}(s)\big)\;ds\label{dpg:1}\\
	& = \int_{s}\mu(s)\;\nabla_{\boldsymbol{\theta}} \boldsymbol{\pi}_{\boldsymbol{\theta},\boldsymbol{\phi}}(s)\;\nabla_{\boldsymbol{a}}\;Q^{\boldsymbol{\pi}_{\boldsymbol{\theta},\boldsymbol{\phi}}}\big(s, \boldsymbol{a}\big)\,\big|_{\boldsymbol{a}=\boldsymbol{\pi}_{\boldsymbol{\theta},\boldsymbol{\phi}}(s)}\;ds\label{dpg:2}\\
	& \approx \mathbb{E}_{(s, \boldsymbol{a})\sim\mathcal{B}}\bigg[\;\nabla_{\boldsymbol{\theta}} \boldsymbol{\pi}_{\boldsymbol{\theta},\boldsymbol{\phi}}(s)\;\nabla_{\boldsymbol{a}}\;Q^{\eta_{k+1}}\big(s, \boldsymbol{a}\big)\,\big|_{\boldsymbol{a}=\boldsymbol{\pi}_{\boldsymbol{\theta},\boldsymbol{\phi}}(s)}\;\bigg].\label{dpg:3}
	\end{align}
	Eq. (\ref{dpg:1}) and Eq. (\ref{dpg:2}) are based on deterministic policy gradient theorem \cite{silver2014deterministic}, Eq. (\ref{dpg:3}) is the practical usage with experience buffer $\mathcal{B}$. Thus the gradient of $J^R(\boldsymbol{f},\boldsymbol{g})$ w.r.t. $\theta_i$ is
	\begin{align*}
	\nabla_{\theta_i}J^R({\boldsymbol{f}}, \boldsymbol{g}) \approx \mathbb{E}_{(s, \boldsymbol{a})\sim\mathcal{B}}\bigg[\;\nabla_{\theta_i}\boldsymbol{\pi}_{\boldsymbol{\theta},\boldsymbol{\phi}}(s)\;\nabla_{\boldsymbol{a}}\;Q^{\eta_{k+1}}\big(s, \boldsymbol{a}\big)\,\big|_{\boldsymbol{a}=\boldsymbol{\pi}_{\boldsymbol{\theta},\boldsymbol{\phi}}(s)}\;\bigg].
	\end{align*}
\end{proof}
Based on the theorem, we can show that gradients are indeed shared among agents in the training process as discussed in Section \ref{sec:framework} by following derivations. We take the stochastic version as an example. Let $\Bar{\boldsymbol{a}}_i=[a_i, a_m, \cdots], \forall{m}\in\mathcal{N}_i$ denote the joint action of agent $i$ and its neighbors, let $\Bar{\boldsymbol{\pi}}_i=[\pi_{\theta_i, \phi_i}, \pi_{\theta_m, \phi_m}, \cdots], \forall{m}\in\mathcal{N}_i$ denote the joint policy of agent $i$ and its neighbors. Then we can derive Eq. (\ref{update_f:stochastic_appendix}) further:
	\begin{align}\label{update_f:stochastic_ind_grad}
	\nabla_{\theta_i}J^R({\boldsymbol{f}}, \boldsymbol{g}) \approx \mathbb{E}_{(s_0, \boldsymbol{b}, \boldsymbol{a})\sim \mathcal{B}}\bigg[\nabla_{\theta_i}\log f_i(b_i|o_i)\;Q^{\eta_{k+1}}(s_0, \boldsymbol{a})+\nabla_{\theta_i}\Bar{\boldsymbol{\pi}}_i(\Bar{\boldsymbol{a}}_i|s_0)\nabla_{\Bar{\boldsymbol{a}}_i}Q^{\eta_{k+1}}(s_0,  \boldsymbol{a})\bigg].
	\end{align}
	Note that the gradient $\nabla_{\theta_i}\Bar{\boldsymbol{\pi}}_i(\Bar{\boldsymbol{a}}_i|s_0)$ can be realized via reparameterization trick \cite{kingma2014autoencoding}. Thus each column $m$ of matrix $\nabla_{\theta_i}\Bar{\boldsymbol{\pi}}_i(\Bar{\boldsymbol{a}}_i|s_0)$ represents the gradient $\pi_{\theta_m, \phi_m}$ w.r.t. $\theta_i$, and gradients of all agents ($\mathcal{N}_i\cup\{i\}$) are aggregated to $\theta_i$ with weight $\nabla_{\Bar{\boldsymbol{a}}_i}Q^{\eta_{k+1}}(s_0,  \boldsymbol{a})$.

\subsection{Proof of Theorem \ref{main_theorem}}\label{detail_proof:main_theorem}
To establish the proof, we introduce two mild assumptions\footnote{These assumptions are commonly adopted in existing works \cite{tessler2018reward, ntk}.}, including that the space $\Phi$ is compact and convex, and that $\mathcal{L}_j(\boldsymbol{\phi}; \boldsymbol{\theta}_{k+1})$ is $L_j$-smooth 
w.r.t. $\boldsymbol{\phi}$, $\forall{j}\in[M]$, with $L_{\max}$ denoting $\max\{L_1,\cdots,L_M\}$. Now we restate Theorem \ref{main_theorem}:
\begin{theorem*}
	Let $\boldsymbol{\phi}$ be updated with the exact constraint violation losses given $\boldsymbol{\theta}_{k+1}$. That is, in each iteration $w$ of Alg. \ref{updating_alg}, $\boldsymbol{\varphi}_{w+1}$ is set as $\Gamma_{\Phi}\big[\boldsymbol{\varphi}_w-\tau\cdot\textnormal{Clip}(\nabla_{\boldsymbol{\phi}}\mathcal{L}_{j^*}(\boldsymbol{\varphi}_w; \boldsymbol{\theta}_{k+1}))\big]$ with $j^*=\arg\max_{j\in[M]}\mathcal{L}_j(\boldsymbol{\varphi}_w; \boldsymbol{\theta}_{k+1})$. Then, for any $\epsilon>0$ and $j\in[M]$, if both $\tau L_{max}$ and $\tau G^2$ are sufficiently small, $\boldsymbol{\varphi}_w$ will converge in $H\leq \frac{\min_{\boldsymbol{\phi}\in\mathcal{X}}||\boldsymbol{\phi}_{k}-\boldsymbol{\phi}||^2}{2\tau\epsilon}$ steps to the region
\begin{align*}
C_k\leq\mathcal{L}_{j}(\boldsymbol{\phi}; \boldsymbol{\theta}_{k+1})\leq C_k+\frac{2\epsilon+\tau G^2}{2F(H)},
\end{align*}
where the set $\mathcal{X}=\arg\min_{\boldsymbol{\phi}\in\Phi}\mathcal{L}_{j^*}(\boldsymbol{\phi}; \boldsymbol{\theta}_{k+1})$, the value $C_k=\min_{\boldsymbol{\phi}\in\Phi}\mathcal{L}_{j^*}(\boldsymbol{\phi}; \boldsymbol{\theta}_{k+1})$, and $F(H)=\min\big(1,\frac{G}{||\nabla_{\boldsymbol{\phi}}\mathcal{L}_{j^*}(\boldsymbol{\varphi}_{{H}}; \boldsymbol{\theta}_{k+1})||}\big)$.
\end{theorem*}
\begin{proof}

We simplify $\mathcal{L}_j(\boldsymbol{\varphi}; \boldsymbol{\theta}_{k+1})$ as $\mathcal{L}_{j}(\boldsymbol{\varphi})$. Define $\boldsymbol{\varphi}^*=\arg\min_{\boldsymbol{\varphi}\in\mathcal{X}}||\boldsymbol{\varphi}_{0}-\boldsymbol{\varphi}||^2$. Note that for any $\boldsymbol{\varphi}\in\Phi$, we have $\mathcal{L}_{j^*}(\boldsymbol{\varphi}^*)\leq\mathcal{L}_{j^*}(\boldsymbol{\varphi})$. Now let us begin the analysis.  Applying non-expansive property of the projection to the updating recursion, we have:
\begin{align*}
    ||\boldsymbol{\varphi}_{w+1}-\boldsymbol{\varphi}^*||^2 & \leq ||\boldsymbol{\varphi}_{w}-\tau\cdot \text{Clip}\big(\nabla_{\boldsymbol{\phi}}\mathcal{L}_{j^*}(\boldsymbol{\varphi}_w)\big)-\boldsymbol{\varphi}^*||^2
\end{align*}
The $\text{Clip}(\cdot)$ operation has two cases. If $||\nabla_{\boldsymbol{\phi}}\mathcal{L}_{j^*}(\boldsymbol{\varphi}_w)||\leq G$, then the RHS of the previous inequality:
\begin{align*}
    RHS & = ||\boldsymbol{\varphi}_{w}-\boldsymbol{\varphi}^*||^2 + 2\tau(\boldsymbol{\varphi}^*-\boldsymbol{\varphi}_{w})^T\nabla_{\boldsymbol{\phi}}\mathcal{L}_{j^*}(\boldsymbol{\varphi}_w) + \tau^2 ||\nabla_{\boldsymbol{\phi}}\mathcal{L}_{j^*}(\boldsymbol{\varphi}_w)||^2 
\end{align*}
Otherwise if $||\nabla_{\boldsymbol{\phi}}\mathcal{L}_{j^*}( \boldsymbol{\varphi}_w)||> G$, then the RHS of the previous inequality:
\begin{align*}
    RHS & = ||\boldsymbol{\varphi}_{w}-\boldsymbol{\varphi}^*||^2 + \frac{2\tau G\cdot(\boldsymbol{\varphi}^*-\boldsymbol{\varphi}_{w})^T}{||\nabla_{\boldsymbol{\phi}}\mathcal{L}_{j^*}(\boldsymbol{\varphi}_w)||}\cdot\nabla_{\boldsymbol{\phi}}\mathcal{L}_{j^*}(\boldsymbol{\varphi}_w) + \tau^2 G^2
\end{align*}
Thus we can conclude that
\begin{equation}\label{proof:eq_1}
    ||\boldsymbol{\varphi}_{w+1}-\boldsymbol{\varphi}^*||^2 \leq
  ||\boldsymbol{\varphi}_{w}-\boldsymbol{\varphi}^*||^2 + 2\tau\cdot F(w) \cdot\big((\boldsymbol{\varphi}^*-\boldsymbol{\varphi}_{w})^T\nabla_{\boldsymbol{\phi}}\mathcal{L}_{j^*}(\boldsymbol{\varphi}_w)\big) + \tau^2 G^2\,,
\end{equation}
where $F(w)=\min(1, \frac{G}{||\nabla_{\boldsymbol{\phi}}\mathcal{L}_{j^*}(\boldsymbol{\varphi}_w)||})$. According to Assumption 2, since $\mathcal{L}_{j^*}$ is $L_{j^*}$-smooth, it follows that:
\begin{align*}
    \mathcal{L}_{j^*}(\boldsymbol{\varphi}^*)&\geq \mathcal{L}_{j^*}(\boldsymbol{\varphi}_w)-(\boldsymbol{\varphi}_w-\boldsymbol{\varphi}^*)^T\nabla_{\boldsymbol{\phi}}\mathcal{L}_{j^*}(\boldsymbol{\varphi}_w)-\frac{L_{j^*}}{2}||\boldsymbol{\varphi}_{w}-\boldsymbol{\varphi}^*||^2.
\end{align*}
then we have
\begin{equation}\label{proof:eq_2}
    (\boldsymbol{\varphi}^*-\boldsymbol{\varphi}_w)^T\nabla_{\boldsymbol{\phi}}\mathcal{L}_{j^*}(\boldsymbol{\varphi}_w)\leq\mathcal{L}_{j^*}(\boldsymbol{\varphi}^*)-\mathcal{L}_{j^*}(\boldsymbol{\varphi}_w)+\frac{L_{j^*}}{2}||\boldsymbol{\varphi}_{w}-\boldsymbol{\varphi}^*||^2,
\end{equation}
Bring inequality (\ref{proof:eq_2}) into (\ref{proof:eq_1}), it follows that:
{\small
\begin{align*}
    ||\boldsymbol{\varphi}_{w+1}-\boldsymbol{\varphi}^*||^2 & \leq ||\boldsymbol{\varphi}_{w}-\boldsymbol{\varphi}^*||^2 + 2\tau\cdot F(w)\cdot \bigg(\mathcal{L}_{j^*}(\boldsymbol{\varphi}^*)-\mathcal{L}_{j^*}(\boldsymbol{\varphi}_w)+\frac{L_{j^*}}{2}||\boldsymbol{\varphi}_{w}-\boldsymbol{\varphi}^*||^2\bigg) + \tau^2 G^2 \\
    & = (\tau\cdot F(w)\cdot L_{j^*}+1)||\boldsymbol{\varphi}_{w}-\boldsymbol{\varphi}^*||^2 + \tau \Bigg(2 F(w)\cdot \bigg(\mathcal{L}_{j^*}(\boldsymbol{\varphi}^*)-\mathcal{L}_{j^*}(\boldsymbol{\varphi}_w)\bigg)+ \tau G^2\Bigg)\\
    & \leq (\tau\cdot F(w)\cdot L_{max}+1)||\boldsymbol{\varphi}_{w}-\boldsymbol{\varphi}^*||^2 + \tau\Bigg(2 F(w)\cdot \bigg(\mathcal{L}_{j^*}(\boldsymbol{\varphi}^*)-\mathcal{L}_{j^*}(\boldsymbol{\varphi}_w)\bigg)+ \tau G^2\Bigg),
\end{align*}
}
where $*$ is by the definition of $H_w$. Denote $\alpha_w=2 \cdot F(w)\cdot \bigg(\mathcal{L}_{j^*}(\boldsymbol{\varphi}^*)-\mathcal{L}_{j^*}(\boldsymbol{\varphi}_w)\bigg)+ \tau G^2$. With $\tau L_{max}$ is sufficiently small, thus writing recursively the previous expression yields
\begin{equation}\label{itr_upb}
\begin{aligned}
    ||\boldsymbol{\varphi}_{w+1}-\boldsymbol{\varphi}^*||^2 & \leq ||\boldsymbol{\varphi}_{0}-\boldsymbol{\varphi}^*||^2+\tau\cdot\sum^w_{v=0}\alpha_v
\end{aligned}
\end{equation}
Now we analyze the expression for $\alpha_w$. Since $\mathcal{L}_{j^*}(\boldsymbol{\varphi}^*)\leq\mathcal{L}_{j^*}(\boldsymbol{\varphi}_w)$ and $\tau G^2$ is sufficiently small, thus when $\boldsymbol{\varphi}_w$ is not close to $\boldsymbol{\varphi}^*$, $\alpha_w$ is negative. Then with Eq. (\ref{itr_upb}), we know that the distance between $\boldsymbol{\varphi}_w$ and $\boldsymbol{\varphi}^*$ is decreased gradually. When $\boldsymbol{\varphi}_w$ finally converges into neighbor of $\boldsymbol{\varphi}^*$, formally, for any $\epsilon>0$, when $\alpha_w\geq -2\epsilon$, we have that:
\begin{align*}
    \mathcal{L}_{j^*}(\boldsymbol{\varphi}_w)-\mathcal{L}_{j^*}(\boldsymbol{\varphi}^*)\leq \frac{2\epsilon+\tau G^2}{2F(w)}.
\end{align*}
Combining the definition of $C_k$, we can derive that
\begin{align*}
    C_k\leq\mathcal{L}_{j^*}(\boldsymbol{\varphi}_w)\leq C_k+\frac{2\epsilon+\tau G^2}{2F(w)}.
\end{align*}
Since $j^*$ achieves the maximum of $\mathcal{L}_{j}(\boldsymbol{\varphi}_w)$, so others $\bigg\{\mathcal{L}_{{j}}(\boldsymbol{\varphi}_w)\;\big|\;j\in[M], j\neq j^*\bigg\}$ also are in the region. Now we are left for finding $H$. Let $H$ be the first that $\alpha_H\geq -2\epsilon$. Formally, $H=\arg\min_w\alpha_w\geq -2\epsilon$. Note that the LHS of Eq. (\ref{itr_upb}) is always positive, so we have
\begin{align}
    0&\leq||\boldsymbol{\varphi}_{0}-\boldsymbol{\varphi}^*||^2+\tau\cdot\sum^H_{v=0}\alpha_v\\&\leq ||\boldsymbol{\varphi}_{0}-\boldsymbol{\varphi}^*||^2 + \tau\cdot \sum_{v=0}^H\cdot(-2\epsilon)\label{appendix:derive_w}
\end{align}
Rearrange Eq. (\ref{appendix:derive_w}), we get that:
\begin{equation}\label{find_w3}
\begin{aligned}
    H\leq \frac{||\boldsymbol{\varphi}_{0}-\boldsymbol{\varphi}^*||^2}{2\tau\epsilon}.
\end{aligned}
\end{equation}
Thus we completes the proof.
\end{proof}

%% file: appendix_exp.tex
\subsection{Cooperative Treasure Collection (CTC)}\label{CTC_details}
\textbf{CTC}: CTC is extended based on the Cooperative Treasure Collection \cite{MAAC} environment. In CTC, 3 "treasure hunters" and 1 "treasure bank" work cooperatively to collect and store treasures. The role of hunters it to collect treasures and then store them into the bank. There are 3 treasures in the map and will get re-spawn randomly once collected. The role of bank is to store treasures from hunters. Hunters and bank can select a coordination point within a square box as actions to move their positions. In CTC-fair, agents' observations contain others' positions with respect to their own, while in CTC-safe, agents additionally have observations that tell the distance of the unsafe regions with respect to their own. Hunters will receive reward for successful collection of treasures and get punished for colliding with other hunters. Besides, each agent also receives reward that is positively correlated with the amount of treasures stored in banks. The team-average immediate reward is just the average immediate reward of all agents. CTC-safe and CTC-fair consider different types of costs, which are shown below. 

\textbf{CTC-safe}: In CTC-safe, we add unsafe regions into the map. These unsafe regions can be regarded as different objects in real world, for example, deep traps or shallow puddles. In experiment 3 unsafe regions are added and each has different diameters and constraint bounds. The immediate cost $j\in[3]$ has following expression:
\begin{equation}
c^j_i(s_t, \boldsymbol{a_t})=
\begin{cases}
  1, & \text{if agent }i\; \text{in unsafe region }j,\\
  0, & \text{otherwise},
\end{cases}
\end{equation}
and $c^j(s_t, \boldsymbol{a_t})=\frac{1}{4}\sum_{i\in[4]} c^j_i(s_t, \boldsymbol{a_t})$ reflects the team-average immediate cost of $j$.

\textbf{CTC-fair}: In CTC-fair, we consider unfairness as cost. Unfairness considers the max difference of the accumulated distance travelled by agents: let $d_i(s_t, \boldsymbol{a_t})=\sum_{t'=0}^{t}||\,a^i_{t'}\,||$ denote the accumulated distance travelled by agent $i$, then the immediate cost has following expression:
\begin{equation}
c_i(s_t, \boldsymbol{a_t})=\max_{i'}d_{i'}-\min_{i'}d_{i'},
\end{equation}
and thus their average $c(s_t, \boldsymbol{a_t})=\frac{1}{4}\sum_{i\in[4]} c_i(s_t, \boldsymbol{a_t})$ also evaluates the max difference of the accumulated distance among agents. 

\subsection{Constrained Directional Sensor Network (CDSN)}
\textbf{DSN}: We extends the original DSN environment \cite{xu2020learning} by considering continuous action space. 5 sensors are depolyed and 4 objects move around in the system. Sensors only have limited field of view and they take actions to adjust angles to capture more objects. The action space we considered is set as $[-5, 5]$. Two immediate reward is returned at each step: the global coverage ratio calculates the number of objects captured versus all objects in the system; the individual reward calculates the objects captured by one sensor (angle and distance related).

\textbf{CDSN}: We consider operational cost to adjust the sensor angle. More specifically, for each agent $i$, immediate cost $c_i(s_t, \boldsymbol{a_t})=|a^i_t|$. The average cost $c(s_t, \boldsymbol{a_t})=\frac{1}{5}\sum_{i\in[5]} c_i(s_t, \boldsymbol{a_t})$ reflects the immediate average operational cost.

\subsection{Constrained Large-scale Fleet Management (CLFM)}\label{clfm_details}
\textbf{LFM}: LFM studies relocating idle drivers in online-hailing platforms\footnote{Examples include Didi Chuxing (http://www.didichuxing.com/en/) and Uber (https://www.uber.com/).} in a distributed way. The data provided by KDD Cup 2020\footnote{See https://outreach.didichuxing.com/competition/kddcup2020/ for more information.} includes approximate 1 million orders from November 1 to November 30, 2016 and hexagon grid data from Chengdu city, China\footnote{Data source: DiDi Chuxing GAIA Open Dataset Initiative, see https://gaia.didichuxing.com.}. The orders data contains basic information of orders, such as origin location and time step, destination location and time step, duration and fee, etc. The hexagon grid data contains the longitudes and latitudes of six vertices of each grid, with each covering approximately 1 square kilometers. In simulation, we selected orders in the urban 103 grids, and start from 6 AM to 12 PM. Our simulator is mainly based on the grid simulator designed by \cite{efficient_fm}. In the simulator there are 500 drivers and 36 time steps in total, each step with 10 minutes. At each time step, the simulator works as follows:
\begin{itemize}[leftmargin=*]
	\item Simulator firstly loads orders and removes unserved orders of last time step.
	\item Idle drivers (i.e., those who are not serving orders) make decisions for reposition: they firstly generate a 3 dimensional vector representing the weight for repositioning. Score for each candidate grid (includes neighboring grids and current grid) is given by the product of the weight and the feature of that grid (each grid has 3 dimensional features, including the number of idle drivers in grid, the number of orders in grid and the time step). Finally drivers sample out a destination based on the normalized scores among the scores. Drivers will take one time step to accomplish reposition.
	\item Orders get dispatched to idle drivers by order dispatch algorithms. Since we do not focus on order dispatching algorithm in this paper, we just randomly dispatch orders onto agents for simplicity, while we state that our simulator can work with any standard order dispatching algorithms. Drivers who are successfully dispatched with orders will reposition to the destination grid by grid, and cannot serve other orders during the process.
	\item For idle drivers, they get 0 immediate reward, while for order-serving drivers, when order is picked-up, i.e., drivers reach to the passengers and start to serve, they will be rewarded with the fee of the order.
\end{itemize}

\textbf{CLFM}: CLFM considers two types of cost. The first considers the city-wide demand-supply gap, which is reflected by the KL-divergence between the distribution of idle drivers and the distribution of orders. At each time step, each agent will receive the identical KL-divergence as the immediate cost. The other cost evaluates the unfairness of drivers' accumulated income. Denote $e^t_i=\sum_{t'=0}^tr^{t'}_i(s_{t'}, \boldsymbol{a_{t'}})$, thus each agent will receive 
\begin{equation}
c^{2}_i(s_t, \boldsymbol{a_t})=(e^t_i-\Bar{e^t})^2,
\end{equation}
where $\Bar{e^t}=\frac{1}{N}\sum_{i\in[N]}e^t_i$. Thus $c^{2}(s_t, \boldsymbol{a_t})=\frac{1}{N}\sum_{i\in[N]}c^{2}_i(s_t, \boldsymbol{a_t})$.

\subsection{Practical Implementations}\label{practical_implemetation}
In this section we introduce the practical implementations of DeCOM, including the evaluation of constraints violation and the hyper-parameter settings.

\textbf{Evaluation of Constraints Violation.} Recall section \ref{optimize_g}, Eq. (\ref{phi_loss:ideal}) evaluates the constraint violation based on $J^{C_j}(\boldsymbol{f}, \boldsymbol{g})$, which is further approximated by Eq. (\ref{phi_loss:prac}), $\forall{j}\in[M]$. The action value function $Q^{\zeta_j, k+1}(s_0, \boldsymbol{a})$ represents the expected team-average cost $j$ given initial state $s_0$ and action $\boldsymbol{a}$, which is updated iteratively by minimizing the TD error \cite{sutton_td}. However in practice we found that $Q^{\zeta_j, k+1}(s_0, \boldsymbol{a})$ can be biased at the beginning stage of training, mainly because the long-term value at the initial time step cannot be approximated accurately without accurate approximation of later time steps (which is also pointed out by \cite{sutton2018reinforcement}: TD methods update the estimates of the value functions via bootstraps of previous estimates on sampled transitions). To speed up training and obtain accurate evaluation of constraints violation as possible, we turn to evaluate constraints violation at each time step instead of merely at the initial time step. 

We adopt the observations proposed in a recent work \cite{pmlr-v119-satija20a} to achieve the goal. When the \textit{CCMG} is episodic, i.e., the game only has $T\in\mathbb{Z}^+<\infty$ time steps and $\gamma=1$, then at each time step $t\in[0, T]$, the long term cost $j\in[M]$ at the initial time step in (\ref{original_problem}) can be decomposed as follows:
\begin{equation}\label{bvf_expression}
    \begin{aligned}
    \mathbb{E}\bigg[\sum_{t'=0}^t c^j(s_{t'}, \boldsymbol{a}_{t'})\;|\;s_0, \boldsymbol{\pi}\bigg] + \mathbb{E}\bigg[\sum_{t'=t}^T c^j(s_{t'}, \boldsymbol{a}_{t'})\;|\;s_0, \boldsymbol{\pi}\bigg] - \mathbb{E}\bigg[c^j(s_{t}, \boldsymbol{a}_{t})\;|\;s_0, \boldsymbol{\pi}\bigg]\leq D_j.
\end{aligned}
\end{equation}
In practice, at episode $k$ of Alg. \ref{decom_alg}, when we have a mini-batch of samples $\mathcal{B}$, the constraints violation at each time step $t\in[0, T]$ can be approximated empirically via:
\begin{equation}\label{cv_eachstep}
    \begin{aligned}
    \Tilde{\mathcal{L}}^{t}_j(\boldsymbol{\phi};\boldsymbol{\theta}_{k+1}) &= \max\Big(0, \mathbb{E}_{s_{t'}, \boldsymbol{a}_{t'}\sim\mathcal{B}}\bigg[\sum_{t'=0}^t c^j(s_{t'}, \boldsymbol{a}_{t'})\bigg]\\&+\mathbb{E}_{s_{t}, \boldsymbol{a}_{t}\sim\mathcal{B}}\bigg[Q^{\zeta_j, k+1}(s_{t}, \boldsymbol{a}_{t})\bigg]-\mathbb{E}_{s_{t}, \boldsymbol{a}_{t}\sim\mathcal{B}}\bigg[ c^j(s_{t}, \boldsymbol{a}_{t})\bigg] -D_j\Big)^2, \forall{j}\in [M].
    \end{aligned}
\end{equation}
Note that $\Tilde{\mathcal{L}}^{0}_j$ in Eq.(\ref{cv_eachstep}) is just $\Tilde{\mathcal{L}}_j$ in Eq.(\ref{phi_loss:prac}). In practice, to avoid the possible sub-optimality gap of Eq. (\ref{bvf_expression}) \cite{pmlr-v119-satija20a}, we select the maximum of $\Tilde{\mathcal{L}}^{0}_j(\boldsymbol{\phi};\boldsymbol{\theta}_{k+1})$ and $\frac{1}{T}\sum_t\Tilde{\mathcal{L}}^{t}_j(\boldsymbol{\phi};\boldsymbol{\theta}_{k+1})$ as the evaluation of constraint violation of cost $j$.

\textbf{Hyper-Parameter Settings.} 
All experiments were done on Intel(R) Xeon(R) Silver 4116 CPU. Next we introduce the hyper-parameters for different environments respectively.

In CTC-safe and CTC-fair, base policy and perturbation policy networks contain \textit{linear} layers and \textit{leaky relu} activation functions. The final output is activated by \textit{Tanh} to keep each dimension of action in $(-1, 1)$. During training procedure, noise generated by Ornstein-Uhlenbeck process is added onto action for exploration \cite{lillicrap2019continuous}. Reward critic and cost critic (CTC-safe:only one cost critic is trained since the 3 costs have the same orders of magnitude, and they are distinguished by one-hot index input.) networks also have \textit{linear} layers and \textit{leaky relu} activation function. There are total $10^5$ episodes and each episode has 25 time steps. Buffer stores the latest $10^6$ experience tuples and training is conducted every 12 episode. Mini-batches have size 1024. All networks' parameters are updated by \textit{Adam}. Base policy, perturbation policy, reward critic, cost critic have learning rate 0.001, 0.003, 0.001, 0.003 respectively. Discount factor for reward is set as 0.99. The maximal gradient norm $G$ is set as 0.5. $\lambda$ who controls the magnitude of perturbation is set as 1 in CTC-safe and 0.01 in CTC-fair. Target networks of base policy and reward critic are updated with 0.01. Target networks of perturbation policy and cost critic are updated with rate 0.05 at initial, and the rate decreases to 0.01 gradually. We test $W$ (the $W$ in Alg.\ref{updating_alg}) from set $\{1, 2, 3\}$ with the consideration of computational efficiency and set $W=1$ in experiment due to its better performance. 3 random seeds are trained independently and test results are the average results of 100 episodes. 

In CDSN, networks of base policy, perturbation policy, reward critic and cost critic are composed by \textit{linear} layers and \textit{elu}, \textit{tanh} activation functions. Base policy generates mean and variance of a Gaussian distribution to sample base actions. Perturbation policy use Ornstein-Uhlenbeck process to generate noise for exploration. There are total $30000$ episodes and each episode has 101 time steps. Buffer stores the latest 10 episode's experience tuples and training is conducted at every 10 episode. All networks' parameters are updated by \textit{Adam}. Base policy, perturbation policy, reward critic, cost critic have learning rate 0.0005, 0.0003, 0.0001, 0.0001 respectively. Discount factor for reward is set as 0.99. The maximal gradient norm $G$ is set as 0.5. $\lambda$ who controls the magnitude of perturbation is set as 1. Target networks of reward critic and cost critic are updated with 0.05. Target networks of base policy and perturbation policy are updated with rate 0.03 and 0.01 respectively. $W$ is set as 1. 5 random seeds are trained independently and test results are the average results of 100 episodes.

In CLFM, networks of base policy, perturbation policy, reward critic and cost critics are composed by \textit{linear} layers and \textit{elu} activation functions. Base policy generates mean and variance of a Gaussian distribution to sample base actions. Perturbation policy also use Ornstein-Uhlenbeck process to generate noise for exploration. In CLFM, 2 cost critics are trained since their costs have large difference in orders of magnitude. There are total $2000$ episodes and each episode has 36 time steps. Buffer stores the latest episode's experience tuples and training is conducted at every episode. All networks' parameters are updated by \textit{Adam}. Base policy, perturbation policy, reward critic, cost critic have learning rate 0.001, 0.0003, 0.0001, 0.0001 respectively. Discount factor for reward is set as 0.99. The maximal gradient norm $G$ is set as 0.5. $\lambda$ who controls the magnitude of perturbation is set as 1. Target networks of reward critic and cost critics are updated with 0.0025. Target networks of base policy and perturbation policy are updated with rate 0.1 and 0.03 respectively. We test $W$ from set $\{1, 2\}$ and set $W=1$ due to its better performance. 3 random seeds are trained independently and test results are the average results of the final 100 episodes.


\subsection{Additional Results}\label{exp_details}
\subsubsection{Different $\lambda$}
In Table \ref{ctc-table}, the $\lambda$ of CTC-fair environment is set as 0.01. We test CTC-fair further with different $\lambda$ to show the significant difference in both the reward and the unfairness, as shown in Table \ref{diff:lambda}. As $\lambda$ increases, the reward decreases and unfairness increases in general. This phenomenon is reasonable to some extent since in CTC-fair, larger weight in the perturbation policy will changes the base action with a large amount, which can have larger difference between trajectories more easily, thus it has worse performance in the unfairness metric. In addition, action in CTC environment is bounded ([-1, 1] in each dimension), and those who violates the bound will be clipped. Therefore larger weights can easily violate the bound and make actions become meaningless, which leads to bad performance in reward. We can expect that, performance in CTC-safe can be improved with smaller $\lambda$.

\subsubsection{Training Curves}
Fig. \ref{train:CTC-safe} presents the training curves in CTC-safe environment. DeCOM and its variant DeCOM-I, DeCOM-N all satisfy constraints finally, FP-0 has the highest reward but violates the constraints. La is unstable and hard to converge. Fig. \ref{train:CTC-fair} and \ref{train:cdsn} present the training curves in CTC-fair and CDSN environment. Fig. \ref{train:clfm} presents the training curves in CLFM. DeCOM achieves the highest reward with the lowest demand-supply gap. We can observe that, with slightly more training episodes, DeCOM is likely to satisfy constraint on unfairness as well.
\begin{figure}
    \centering
    \includegraphics[scale=0.4]{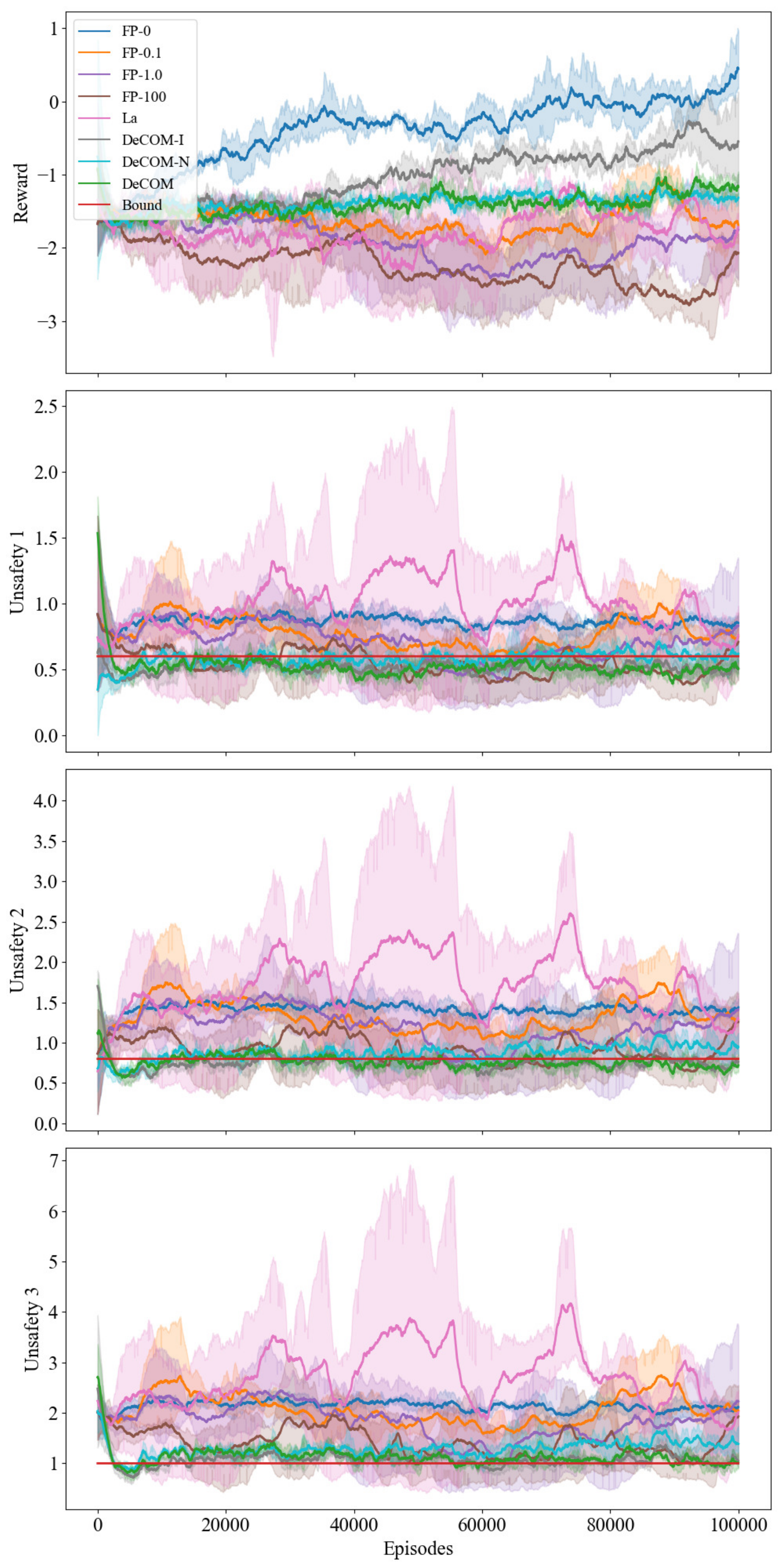}
    \caption{Training Curves of CTC-safe.}
    \label{train:CTC-safe}
\end{figure}

\begin{figure}
    \centering
    \includegraphics[scale=0.4]{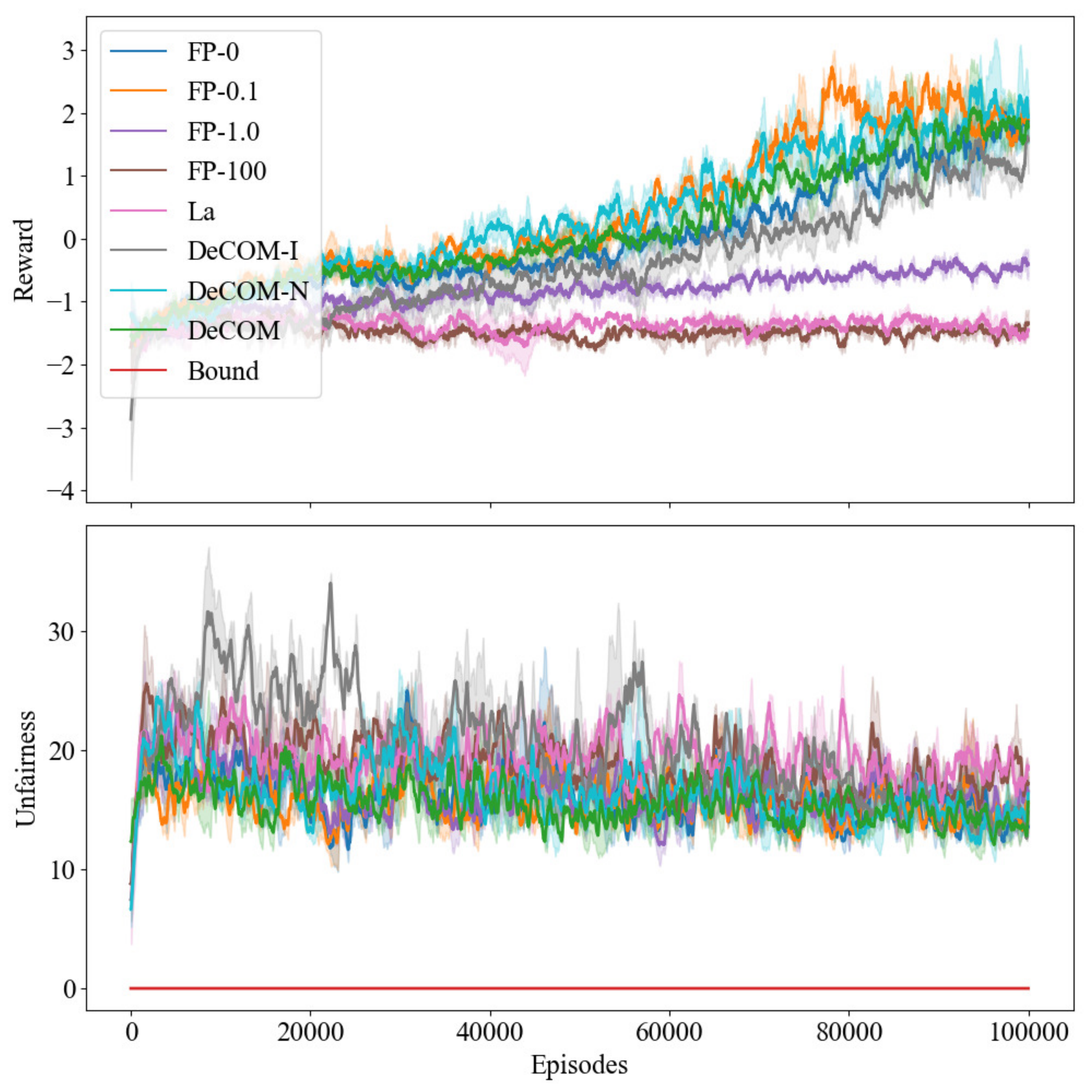}
    \caption{Training Curves of CTC-fair.}
    \label{train:CTC-fair}
\end{figure}
\begin{figure}
	\centering
	\includegraphics[scale=0.4]{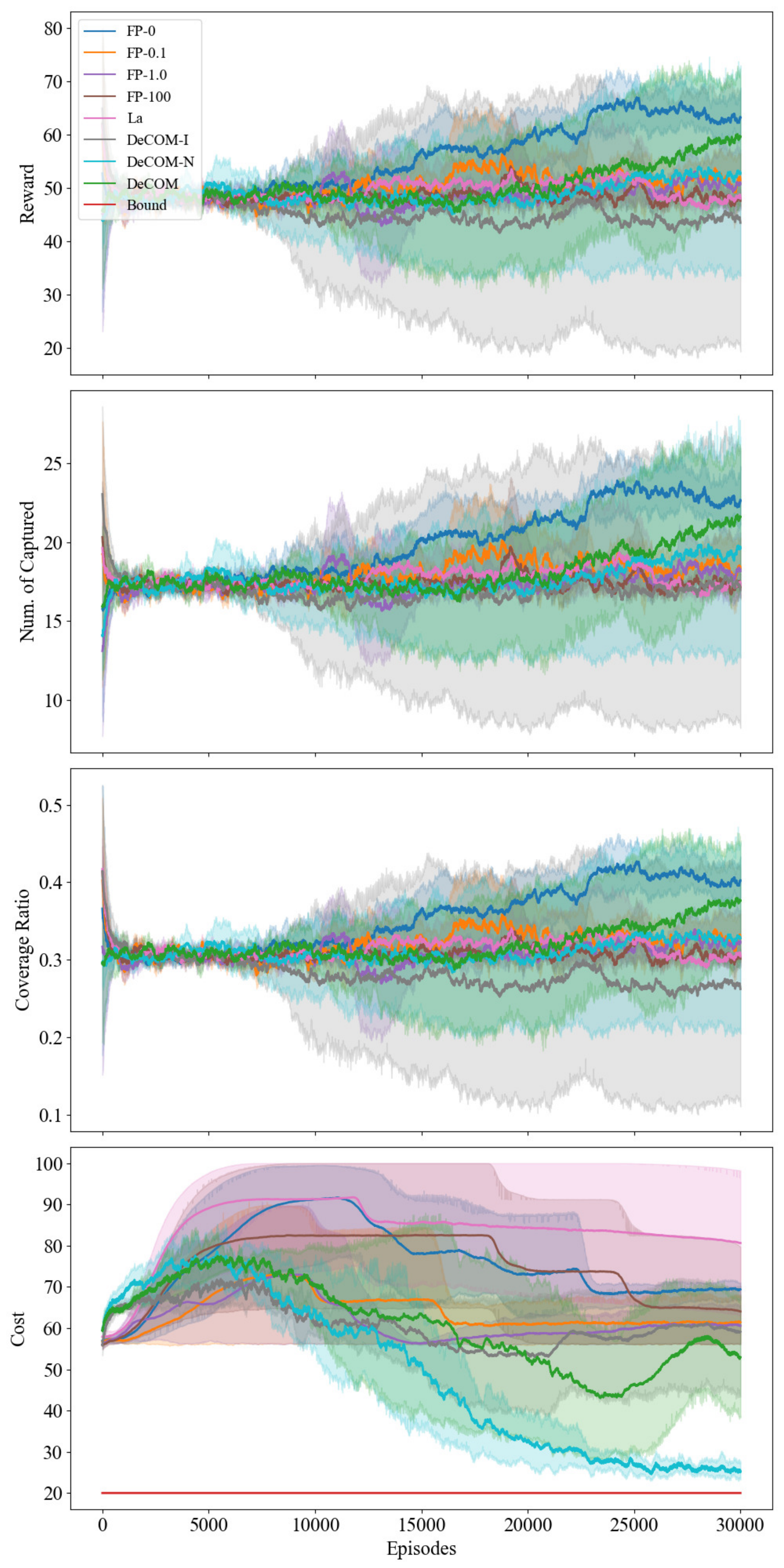}
	\caption{Training Curves of CDSN.}
	\label{train:cdsn}
\end{figure}
\begin{figure}
    \centering
    \includegraphics[scale=0.4]{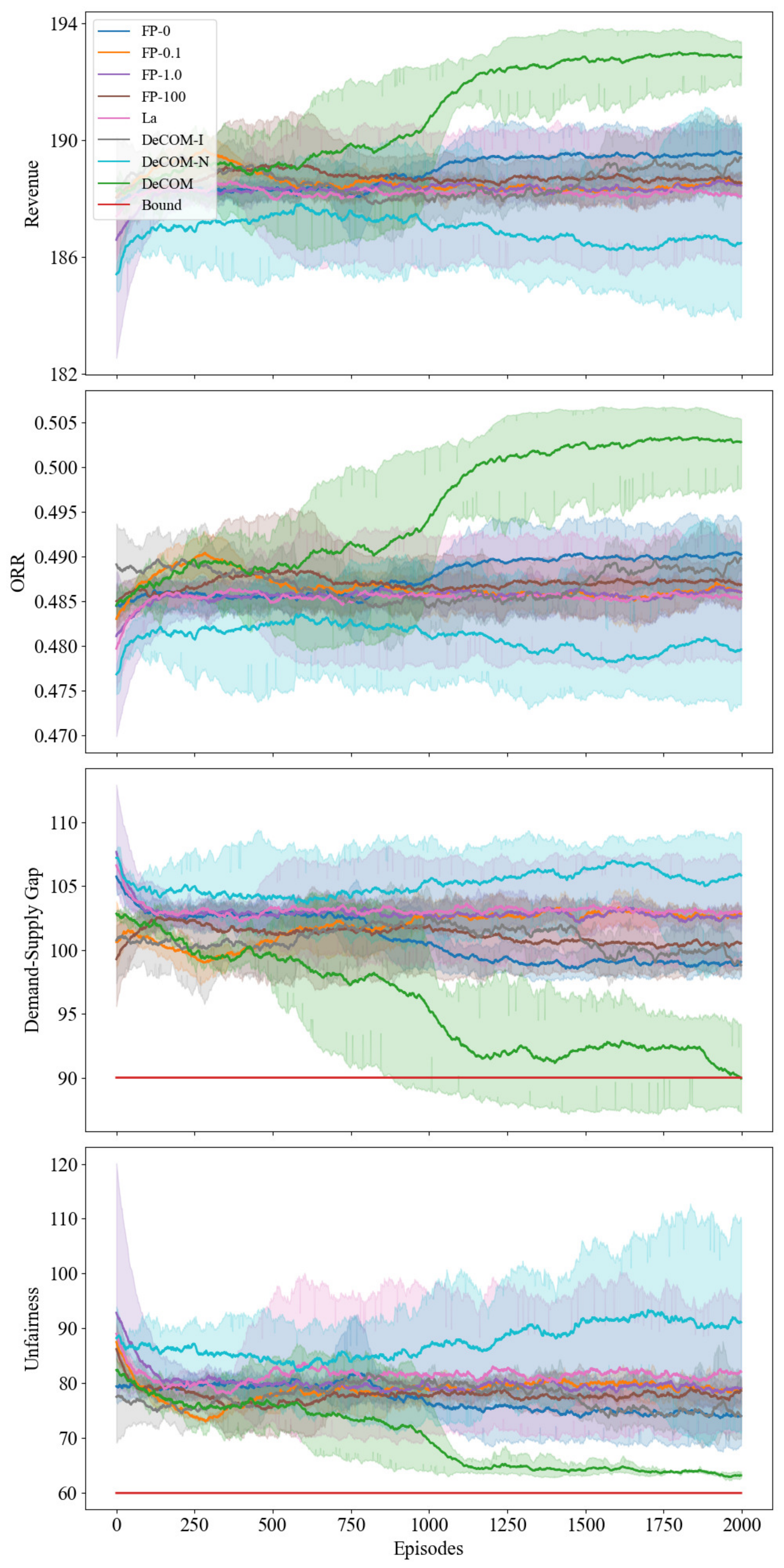}
    \caption{Training Curves of CLFM.}
    \label{train:clfm}
\end{figure}

\begin{table}[]
	\caption{Different $\lambda$ in CTC-fair with DeCOM}
	\label{diff:lambda}
	\centering
		\begin{tabular}{c|cc}
			\hline
			$\lambda$ & \textbf{Reward} & \textbf{Unfairness} \\ \hline
			0.01   & \textbf{3.76}             $\pm$  0.84      & \textbf{8.11}        $\pm$  0.08      \\
			0.1    & 2.50             $\pm$  0.74      & 10.33                $\pm$  2.32      \\
			0.5    & \textbf{4.10}    $\pm$  0.62      & 11.16                $\pm$  1.27      \\
			1      & 0.27             $\pm$  0.29      & 11.08                $\pm$  0.70      \\
			2      & 0.06             $\pm$  0.53      & 15.06                $\pm$  2.22      \\ \hline
		\end{tabular}%
\end{table}